%% file: main.tex
\documentclass[11pt]{article}

\usepackage{fullpage}
\usepackage{natbib}

\usepackage{times}
\usepackage{url}            

\usepackage{nicefrac}       
\usepackage{microtype}      

\usepackage{algorithm}
\usepackage[noend]{algorithmic}
\usepackage{amsmath,amsthm,amsfonts}
\usepackage{amssymb}
\usepackage{hyperref}
\usepackage{color}
\usepackage{mathrsfs}
\usepackage{enumitem}
\usepackage{bm}
\usepackage{multirow}
\usepackage{booktabs}
\usepackage{makecell}
\usepackage{graphicx}
\usepackage{subfigure}
\usepackage{caption}
\usepackage{thmtools}
\usepackage{thm-restate}

\usepackage{verbatim}
\usepackage{tikz}

\usepackage{xspace}

\newtheorem{theorem}{Theorem}
\newtheorem{lemma}[theorem]{Lemma}

\newtheorem{remark}[theorem]{Remark}

\newtheorem{proposition}[theorem]{Proposition}
\newtheorem{fact}[theorem]{Fact}
\theoremstyle{definition}
\newtheorem{definition}[theorem]{Definition}

\newtheorem{example}[theorem]{Example}

\include{notations}

\DeclareMathAlphabet{\mathpzc}{OT1}{pzc}{m}{it}

\newcommand{\psne}{pure strategy Nash equilibrium\xspace}
\newcommand{\psnes}{pure strategy Nash equilibria\xspace}

\newcommand{\minimax}{minimax point\xspace}
\newcommand{\minimaxs}{minimax points\xspace}
\newcommand{\ttt}[1]{\mathpzc{#1}}
\newcommand{\ett}[1]{\mathpzc{#1}}
\newcommand{\bw}{\bar{\w}}

\newcommand{\dx}{\delta_\x}
\newcommand{\dy}{\delta_\y}
\newcommand{\cW}{\mathcal{W}}

\title{\textbf{What is Local Optimality in Nonconvex-Nonconcave Minimax Optimization?}}

\author{Chi Jin \\ University of California, Berkeley \\ \texttt{chijin@cs.berkeley.edu}
\and 
Praneeth Netrapalli \\ Microsoft Research, India \\ \texttt{praneeth@microsoft.com}
\and
Michael I. Jordan \\ University of California, Berkeley \\ \texttt{jordan@cs.berkeley.edu}}

\begin{document}

\maketitle

\begin{abstract}
Minimax optimization has found extensive application in modern machine learning, in 
settings such as generative adversarial networks (GANs), adversarial training and 
multi-agent reinforcement learning. As most of these applications involve continuous 
nonconvex-nonconcave formulations, a very basic question arises---``what is a proper 
definition of local optima?''

Most previous work answers this question using classical notions of equilibria 
from \emph{simultaneous} games, where the min-player and the max-player act 
simultaneously. In contrast, most applications in machine learning, including 
GANs and adversarial training, correspond to \emph{sequential} games, where the order of which player acts first is crucial (since minimax is in general not equal to maximin due to the
nonconvex-nonconcave nature of the problems). The main contribution of this paper 
is to propose a proper mathematical definition of local optimality for this sequential 
setting---\emph{local minimax}---as well as to present its properties and existence 
results. Finally, we establish a strong connection to a basic local search algorithm---gradient descent ascent (GDA)---under mild conditions, all stable limit points of GDA are exactly local minimax points up to some degenerate points.

\end{abstract}

\vspace{3ex}

\input{intro}

\input{related}

\input{prelim}
\input{result_minimax}
\input{result_gda}

\input{result_oracle}
\input{conclu}
\section*{Acknowledgements}
The authors would like to thank Guojun Zhang and Yaoliang Yu for raising a
technical issue with Proposition \ref{prop:second_nece} in an earlier version
of this paper, Oleg Burdakov for pointing out the related work by Yu. G. Evtushenko,
and Lillian J. Ratliff for helpful discussions.  This work was supported in 
part by the Mathematical Data Science program of the Office of Naval Research 
under grant number N00014-18-1-2764.


\bibliographystyle{plainnat}
\bibliography{minimax}

\newpage
\onecolumn
\appendix
\input{mixed}
\input{relation}
\input{pf_minimax}
\input{pf_prop}
\input{pf_gda}
\input{pf_oracle}

\end{document}

%% file: notations.tex
\newcommand{\defeq}{\mathrel{\mathop:}=}

\newcommand{\vect}[1]{\ensuremath{\mathbf{#1}}}
\newcommand{\mat}[1]{\ensuremath{\mathbf{#1}}}
\newcommand{\dd}{\mathrm{d}}
\newcommand{\grad}{\nabla}
\newcommand{\hess}{\nabla^2}
\newcommand{\pgrad}{\nabla}
\newcommand{\phess}{\nabla^2}
\newcommand{\argmin}{\mathop{\mathrm{argmin}}}
\newcommand{\argmax}{\mathop{\mathrm{argmax}}}

\newcommand{\iprod}[2]{\langle #1, #2 \rangle}

\newcommand{\norm}[1]{\|{#1} \|}

\newcommand{\diag}{\mathrm{diag}}
\renewcommand{\det}{\mathrm{det}}

\newcommand{\trans}{^{\top}}

\renewcommand{\Re}{\mathrm{Re}}
\newcommand{\cP}{\mathcal{P}}

\newcommand{\Z}{\mathbb{Z}}

\newcommand{\R}{\mathbb{R}}

\newcommand{\E}{\mathbb{E}}
\renewcommand{\Pr}{\mathbb{P}}

\newcommand{\A}{\mat{A}}
\newcommand{\B}{\mat{B}}
\newcommand{\C}{\mat{C}}
\renewcommand{\P}{\mat{P}}
\newcommand{\Q}{\mat{Q}}

\newcommand{\I}{\mat{I}}
\newcommand{\J}{\mat{J}}

\newcommand{\U}{\mat{U}}

\newcommand{\X}{\mat{X}}

\newcommand{\w}{\vect{w}}
\newcommand{\x}{\vect{x}}
\newcommand{\xtilde}{\vect{\widetilde{x}}}
\newcommand{\xhat}{\vect{\widehat{x}}}
\newcommand{\yhat}{\vect{\widehat{y}}}
\newcommand{\y}{\vect{y}}
\newcommand{\z}{\vect{z}}
\newcommand{\g}{\vect{g}}
\newcommand{\zero}{\vect{0}}

\newcommand{\cD}{\mathcal{D}}

\newcommand{\nn}{\nonumber}

\newcommand{\cX}{\mathcal{X}}
\newcommand{\cY}{\mathcal{Y}}

\newcommand{\phiell}{\phi_{{1}/{2\ell}}}

%% file: intro.tex

\section{Introduction}

Minimax optimization refers to problems of two agents---one agent tries to minimize the payoff function $f:\cX\times\cY \rightarrow \R$ while the other agent tries to maximize it. Such problems arise in a number of fields, including mathematics, biology, social science, and particularly economics~\citep{myerson2013game}. Due to the wide range of applications of these problems and their rich mathematical structure, they have been studied for several decades in the setting of zero-sum games~\citep{morgenstern1953theory}. In the last few years, minimax optimization has also found significant applications in machine learning, in settings such as generative adversarial networks (GAN)~\citep{goodfellow2014generative}, adversarial training~\citep{madry2017towards} and multi-agent reinforcement learning~\citep{omidshafiei2017deep}. In practice, these minimax problems are often solved using gradient-based algorithms, especially gradient descent ascent (GDA), an algorithm that alternates between a gradient descent step for $\x$ and some number of gradient ascent steps for $\y$. 

A well-known notion of optimality in this setting is that of a \emph{Nash equilibrium}---no player can benefit by changing strategy while the other player keeps hers unchanged. That is, a Nash equilibrium is a point $(\x^\star, \y^\star)$ where $\x^\star$ is a global minimum of $f(\cdot,\y^\star)$ and $\y^\star$ is a global maximum of $f(\x^\star, \cdot)$. In the convex-concave setting, it can be shown that an approximate Nash equilibrium can be  found efficiently by variants of GDA~\citep{bubeck2015convex,hazan2016introduction}. However, most of the minimax problems arising in modern machine learning applications do not have this simple convex-concave structure. Meanwhile, in the general nonconvex-nonconcave setting, one cannot expect to find Nash equilibria efficiently as the special case of nonconvex optimization is already NP-hard. This motivates the quest to find a local surrogate instead of a global optimal point. Most previous work \citep[e.g.,][]{daskalakis2018limit,mazumdar2018convergence,adolphs2018local} studied a notion of \emph{local Nash equilibrium} which replaces all the global minima or maxima in the definition of Nash equilibrium by their local counterparts.

The starting point of this paper is the observation that the notion of local Nash equilibrium is \emph{not} suitable for most machine learning applications of minimax optimization. In fact, the notion of Nash equilibrium (on which local Nash equilibrium is based),
was developed in the context of \emph{simultaneous} games, and so it does not reflect the order between the min-player and the max-player. In contrast, most applications in machine learning, including GANs and adversarial training, correspond to \emph{sequential} games, where one player acts first and the other acts second. When $f$ is nonconvex-nonconcave, $\min_\x\max_\y f(\x, \y)$ is in general not equal to $\max_\y\min_\x f(\x, \y)$; the order of which player acts first is crucial for the problem.
This motivates the question:

\begin{center}
\textbf{What is a good notion of local optimality in nonconvex-nonconcave minimax optimization?}
\end{center}

To answer this question, we start from the optimal solution $(\x^\star, \y^\star)$ for two-player sequential games $\min_{\x}\max_{\y}f(\x, \y)$, where $\y^\star$ is again the global maximum of $f(\x^\star, \cdot)$, but $\x^\star$ is now the global minimum of $\phi(\cdot)$, where $\phi(\x) \defeq \max_{\y \in \cY} f(\x, \y)$. We call these optimal solutions \emph{global minimax points}. The main contribution of this paper is to propose a proper mathematical definition of local optimality for this sequential setting---\emph{local minimax}---a local surrogate for the global minimax points. This paper also presents existence results for this notion of optimality, and establishes several important properties of local minimax points.  These properties naturally reflect the order of which player acts first, and alleviate many of the problematic issues of local Nash equilibria. Finally, the notion of local minimax provides a much stronger characterization of the asymptotic behavior of GDA---under certain idealized parameter settings, all stable limit points of GDA are exactly local minimax points up to some degenerate points. This provides, for the first time, a game-theoretic meaning for all of the stable limit points of GDA.

\subsection{Our contributions} 
To summarize, this paper makes the following contributions.
\begin{itemize}
\item We clarify the difference between several notions of global and local optimality in the minimax optimization literature, in terms of definitions, settings, and properties (see Section \ref{sec:prelims} and Appendix \ref{sec:mixed}).

\item We propose a new notion of local optimality---\emph{local minimax}---a proper mathematical definition of local optimality for the two-player sequential setting. We also present properties of local minimax points and establish existence results (see Section \ref{sec:minimax} and \ref{sec:prop}).

\item We establish a strong connection between local minimax points and the asymptotic behavior of GDA, and provide the first game-theoretic explanation of all stable limit points of GDA, up to some degenerate points (see Section \ref{sec:gda}).

\item We provide a general framework and an efficiency guarantee for a special case where the maximization $\max_\y f(\x, \y)$ can be solved efficiently for any fixed $\x$, or in general when an approximate max-oracle is present (see Appendix \ref{sec:oracle}). 
\end{itemize}

%% file: related.tex
\subsection{Related work}\label{sec:related}
\textbf{Minimax optimization}:
Since the seminal paper of~\cite{neumann1928theorie}, notions of equilibria in games and their algorithmic computation have received wide attention. In terms of algorithmic computation, the vast majority of results focus on
the convex-concave setting~\citep{korpelevich1976extragradient,nemirovsky1978,nemirovski2004prox}. In the context of optimization, these problems have generally been studied in the setting of constrained convex optimization~\citep{bertsekas2014constrained}.
Results beyond convex-concave setting are much more recent. \citet{rafique2018non} and \citet{nouiehed2019solving} consider nonconvex-but-concave minimax problems where for any $\x$, $f(\x,\cdot)$ is a concave function. In this case, they propose algorithms combining approximate maximization over $\y$ and a proximal gradient method for $\x$ to show convergence to stationary points.~\citet{lin2018solving} consider a special case of the nonconvex-nonconcave minimax problem, where the function $f(\cdot,\cdot)$ satisfies a variational inequality. In this setting, they consider a proximal algorithm that requires the solving of certain strong variational inequality problems in each step and show its convergence to stationary points.~\citet{hsieh2018finding} propose proximal methods that asymptotically converge to a \emph{mixed} Nash equilibrium; i.e., a distribution rather than a point.

The most closely related prior work is that of \citet{evtushenko1974some},
who proposed a concept of ``local'' solution that is similar to the local minimax points proposed in this paper. Note, however, that Evtushenko's ``local'' notion is not a truly local property (i.e., cannot be determined just based on the function values in a small neighborhood of the given point). As a consequence, Evtushenko's definition does not satisfy the first-order and second-order necessary conditions of local minimax points (Proposition \ref{prop:first_nece} and Proposition \ref{prop:second_nece}). We defer detailed comparison to Appendix~\ref{sec:relation_earlier}. Concurrent to our work,~\citet{fiez2019convergence} also recognizes the important difference between simultaneous games and sequential games in the machine learning context, and proposes a local notion referred to as \emph{Differential Stackelberg Equilibrium}, which implicitly assumes the Hessian for the second player to be nondegenerate,\footnote{The definition in \citet{fiez2019convergence} implicitly assumes that the best response $r: \cX \rightarrow \cY$ is well-defined by implicit equation $\grad_\y f(x, r(x)) = 0$, and $g(x) := f(x, r(x))$ is differentiable with respect to $x$, conditions which do not always hold even if $f$ is infinitely differentiable. These conditions are typically ensured by assuming $\hess_{\y\y} f \prec \zero$.} in which case it is equivalent to a \emph{strict} local minimax point (defined in Proposition \ref{prop:first_suff}). In contrast, we define a notion of local minimax point in a general setting, including the case in which Hessian matrices are degenerate. Finally, we consider GDA dynamics, which differ from the Stackelberg dynamics considered in~\citet{fiez2019convergence}.

\textbf{GDA dynamics}: There have been several lines of work studying GDA dynamics for minimax optimization. \citet{cherukuri2017saddle} investigate GDA dynamics under some strong conditions and show that the algorithm converges locally to Nash equilibria. \citet{heusel2017gans} and \citet{nagarajan2017gradient} similarly impose strong assumptions in the setting of the training of GANs and show that under these conditions Nash equilibria are stable fixed points of GDA. \citet{gidel2018negative} investigate the effect of simultaneous versus alternating gradient updates as well as the effect of momentum on the convergence in bilinear games.
The analyses most closely related to ours are~\citet{mazumdar2018convergence}
and \citet{daskalakis2018limit}. While~\citet{daskalakis2018limit} study minimax optimization (or zero-sum games),~\citet{mazumdar2018convergence} study a much more general setting of~non-zero-sum games and multi-player games. Both of these papers show that the stable limit points of GDA are not necessarily Nash equilibria.~\citet{adolphs2018local} and \citet{mazumdar2019finding} propose Hessian-based algorithms whose stable fixed points are exactly Nash equilibria. We note that all the papers in this setting use Nash equilibrium as the notion of goodness.

\textbf{Variational inequalities}: Variational inequalities are generalizations of minimax optimization problems. The appropriate generalization of convex-concave minimax problems are known as monotone variational inequalities which have applications in the study of differential equations~\citep{kinderlehrer1980introduction}. There is a large literature on the design of efficient algorithms for finding solutions to monotone variational inequalities~\citep{bruck1977weak,nemirovsky1981,nemirovski2004prox}.

%% file: prelim.tex
\section{Preliminaries}\label{sec:prelims}

In this section, we first introduce our notation and then present definitions and basic results for simultaneous games, sequential games, and general game-theoretic dynamics that are relevant to our work. This paper will focus on two-player zero-sum games. For clarity, we restrict our attention to \emph{pure strategy} games in the main paper, that is, each player is restricted to play a single action as her strategy. We will discuss the relationships of these results to \emph{mixed strategy} games in Appendix \ref{sec:mixed}.


\textbf{Notation.} We use bold upper-case letters $\A, \B$ to denote matrices and bold lower-case letters $\x, \y$ to denote vectors. For vectors we use $\norm{\cdot}$ to denote the $\ell_2$-norm, and for matrices we use $\norm{\cdot}$ and $\rho(\cdot)$ to denote spectral (or operator) norm and spectral radius (largest absolute value of eigenvalues) respectively. Note that these two are in general different for asymmetric matrices.
For a function $f: \R^{d} \rightarrow \R$, we use $\grad f$ and $\hess f$ to denote its gradient and Hessian. For functions of two vector arguments, $f: \R^{d_1}\times \R^{d_2} \rightarrow \R$ , we use $\pgrad_\x f$, $\pgrad_\y f$ and $\phess_{\x\x} f$, $\phess_{\x\y} f$, $\phess_{\y\y} f$ to denote its partial gradient and partial Hessian. We also use $O(\cdot)$ and $o(\cdot)$ notation as follows: $f(\delta) = O(\delta)$ means $\limsup_{\delta \rightarrow 0} |f(\delta)/\delta| \le C $ for some large absolute constant $C$, and $g(\delta) = o(\delta)$ means $\lim_{\delta \rightarrow 0} |g(\delta)/\delta| =0$. For complex numbers, we use $\Re(\cdot)$ to denote its real part, and $|\cdot|$ to denote its modulus. We also use $\cP(\cdot)$, operating over a set, to denote the collection of all probability measures over the set.

\subsection{Simultaneous games}
A \emph{two-player zero-sum} game is a game of two players with a common payoff function $f:\cX\times\cY \rightarrow \R$.
The function $f$ maps the actions taken by both players $(\x, \y) \in \cX \times \cY$ to a real value, which represents the gain of $\y$-player as well as the loss of $\x$-player. We call $\y$ player, who tries to maximize the payoff function $f$, the \emph{max-player}, and $\x$-player the \emph{min-player}. In this paper, we focus on continuous payoff functions $f$, and assume $\cX \subset \R^{d_1}$ and $\cY \subset \R^{d_2}$.

In simultaneous games, the players act simultaneously. That is, each player chooses her action without knowledge of the action chosen by other player. A well-known notion of optimality in simultaneous games is the Nash equilibrium, where no player can benefit by changing strategy while the other player keeps hers unchanged. We formalize this as follows.

\begin{definition}\label{def:pureNash}
Point $(\x^\star, \y^\star)$ is a \textbf{Nash equilibrium} of function $f$, if
for any $(\x, \y)$ in $\cX \times \cY$:
\begin{equation*}
f(\x^\star, \y) \le f(\x^\star, \y^\star) \le f(\x, \y^\star).
\end{equation*}
\end{definition}
That is, $\x^\star$ is a global minimum of $f(\cdot, \y^\star)$ which keeps the action of the $\y$-player unchanged, and $\y^\star$ is a global maximum of $f(\x^\star, \cdot)$ which keeps the action of the $\x$-player unchanged.

Classical work typically focuses on finding Nash equilibria in the setting in which the payoff function $f$ is \emph{convex-concave} (i.e., $f(\cdot, \y)$ is convex for all $\y\in \cY$, and $f(\x, \cdot)$ is concave for all $\x\in \cX$) \citep{bubeck2015convex}. However, in most modern applications in machine learning the payoff $f$ is a \emph{nonconvex-nonconcave} function, and the problem of finding global Nash equilibrium is NP-hard in general. This has led to the study of the following local alternative~\cite[see, e.g.,][]{mazumdar2018convergence,daskalakis2018limit}:
\begin{definition}\label{def:localNash}
Point $(\x^\star, \y^\star)$ is a \textbf{local Nash equilirium} of $f$ if there exists $\delta>0$ such that
for any $(\x, \y)$ satisfying $\norm{\x - \x^\star} \le \delta$ and $\norm{\y - \y^\star} \le \delta$ we have:
\begin{equation*}
f(\x^\star, \y) \le f(\x^\star, \y^\star) \le f(\x, \y^\star).
\end{equation*}
\end{definition}

Local Nash equilibria can be characterized in terms of first-order and second-order conditions.

\begin{proposition}[First-order Necessary Condition]\label{prop:firstNash}
Assuming $f$ is differentiable, 
any local Nash equilibrium satisfies $\grad_\x f(\x, \y) = 0$ and $\grad_\y f(\x, \y) = 0$.
\end{proposition}

\begin{proposition}[Second-order Necessary Condition] \label{prop:secondNash}
Assuming $f$ is twice-differentiable, 
any local Nash equilibrium satisfies $\phess_{\y\y} f(\x, \y) \preceq \zero, \text{~and~} \phess_{\x\x} f(\x, \y) \succeq \zero$.
\end{proposition}

\begin{proposition}[Second-order Sufficient Condition] Assuming $f$ is twice-differentiable, 
any stationary point (i.e., $\nabla f = 0$) satisfying the following condition is a local Nash equilibrium:
\begin{equation} \label{eq:strict_Nash}
\phess_{\y\y} f(\x, \y) \prec \zero, \text{~and~} \phess_{\x\x} f(\x, \y) \succ \zero.
\end{equation}
We also call a stationary point satisfying~\eqref{eq:strict_Nash} a \textbf{strict local Nash equilibrium}.
\end{proposition}

One significant drawback of considering local or global Nash equilibria in nonconvex-nonconcave settings is that they may not exist even for simple well-behaved functions.
\begin{restatable}{proposition}{PROPnoexist}
There exists a twice-differentiable function $f$, where \psnes (either local or global) do not exist.
\end{restatable}

\begin{proof}
Consider a two-dimensional function $f(x, y) = \sin (x+y)$. We have 
$\grad f(x, y) = (\cos(x+y), \cos(x+y))$.
Assuming $(x, y)$ is a local \psne, by Proposition \ref{prop:firstNash} it must also be a stationary point; that is, $x+y = (k+1/2)\pi$ for $k\in\Z$.
It is easy to verify, for odd $k$, $\phess_{xx} f (x, y)= \phess_{yy} f (x, y) = 1 > 0$; for even $k$, 
$\phess_{xx} f (x, y)= \phess_{yy} f (x, y) = -1 < 0$. By Proposition \ref{prop:secondNash}, none of the stationary points is a local \psne.
\end{proof}


\subsection{Sequential games}

In sequential games, there is an intrinsic order such that one player chooses her action before the other one chooses hers. Importantly, the second player can observe the action taken by the first player, and adjust her action accordingly. We would like to emphasize that although many recent works have focused on the simultaneous setting, GAN and adversarial training are in fact sequential games in their standard formulations.

\begin{example}[Adversarial Training]
The target is to train a robust classifier that is robust to adversarial noise. 
The first player picks a classifier, and the second player then chooses adversarial noise to undermine the performance of the chosen classifier.
\end{example}

\begin{example}[Generative Adversarial Network (GAN)]
The target is to train a generator which can generate samples that are similar to real samples in the world.
The first player picks a generator, and the second player then picks a discriminator that is capable of telling the difference between real samples and the samples generated by the chosen generator.
\end{example}

Without loss of generality, in this paper we assume that the min-player is the first player, and the max-player is second. The objective of this game corresponds to following minimax optimization problem:
\begin{equation}\label{eq:minimaxprob}
\min_{\x\in \cX}\max_{\y\in \cY} f(\x, \y).
\end{equation}
where $\min_{\x}\max_{\y}$ already reflects the intrinsic order of the sequential game. While this order does not matter for convex-concave $f(\cdot,\cdot)$, as the minimax theorem~\citep{sion1958general} guarantees that $\min_{\x\in \cX}\max_{\y\in \cY} f(\x, \y) = \max_{\y\in \cY} \min_{\x\in \cX} f(\x, \y)$, for a general nonconvex-nonconcave $f(\cdot,\cdot)$, we have:
\begin{equation*}
\min_{\x\in \cX}\max_{\y\in \cY} f(\x, \y) \neq \max_{\y\in \cY}\min_{\x\in \cX} f(\x, \y).
\end{equation*}
This means that the choice of which player goes first plays an important role in the solution. 

The global solution for Eq.~\eqref{eq:minimaxprob}---a \emph{Stackelberg equilibrium}---is for the second player to always play the maximizer of $f(\x, \cdot)$ given the action $\x$ taken by the first player, and achieve the maximum value $\phi(\x) \defeq \max_{y \in \cY} f(\x, \y)$. Then, the optimal strategy for the first player is to minimize $\phi(\x)$, which gives following definition of global optimality. In this paper, we also call it a global minimax point.

\begin{definition}\label{def:globalminimax}
$(\x^\star, \y^\star)$ is a \textbf{global \minimax}, if for any $(\x, \y)$ in $\cX \times \cY$ we have:
\begin{equation*}
f(\x^\star, \y) \le f(\x^\star, \y^\star) \le \max_{\y'\in \cY} f(\x, \y').
\end{equation*}
\end{definition}

\begin{remark}\label{rm:global}
Equivalently, $(\x^\star, \y^\star)$ is a global \minimax if and only if $\y^\star$ is a global maximum of $f(\x^\star, \cdot)$, and $\x^\star$ is a global minimum of $\phi(\cdot)$ where $\phi(\x) \defeq \max_{y \in \cY} f(\x, \y)$.
\end{remark}
Unlike Nash equilibria, a global \minimax always exists even if $f$ is nonconvex-nonconcave, due to the extreme-value theorem. 
\begin{proposition}\label{prop:minimaxexist}
Assume that function $f: \cX \times \cY \rightarrow \R$ is continuous, and assume that $\cX \subset \R^{d_1}$, $\cY \subset \R^{d_2}$ are compact. Then the global \minimax of $f$ always exists.
\end{proposition}

Finding global minimax points of nonconvex-nonconcave functions is also NP-hard in general. A practical solution is to find a local surrogate. Unfortunately, to the best of our knowledge, there is no formal definition of a local notion of global minimaxity in the literature.

Finally, we note that there is one easier case where the approximate maximization, $\max_{y \in \cY} f(\x, \y)$, can be solved efficiently for any $\x \in \cX$. Then, Eq.~\eqref{eq:minimaxprob} reduces to optimizing $\phi(\cdot)$---a nonsmooth nonconvex function, where efficient guarantees can be obtained (see Appendix \ref{sec:oracle}).




\subsection{Dynamical systems}
One of the most popular algorithms for solving minimax problems is Gradient Descent Ascent (GDA). We outline the algorithm in Algorithm~\ref{algo:GDA}, with updates written in a general form, $\z_{t+1} = \w(\z_t)$, where $\w: \R^d \rightarrow \R^d$ is a vector function. One notable distinction from standard gradient descent is that $\w(\cdot)$ may not be a gradient field (i.e., the gradient of a scalar function $\phi(\cdot)$), and so the Jacobian matrix $\J \defeq \partial \w /\partial \z$ may be asymmetric. This results in the possibility of the dynamics $\z_{t+1} = \w(\z_t)$ converging to a limit cycle instead of a single point. Nevertheless, we can still define fixed points and stability for general dynamics.

\begin{definition}
Point $\z^\star$ is a \textbf{fixed point} for dynamical system $\w$ if $\z^\star = \w(\z^\star)$.
\end{definition}

\begin{definition}[Linear Stability]
For a differentiable dynamical system $\w$, a fixed point $\z^\star$ is a \textbf{linearly stable point} of $\w$ if its Jacobian matrix $\J(\z^\star) \defeq (\partial \w /\partial \z)(\z^\star)$ has spectral radius $\rho(\J(\z^\star)) \le 1$. We also say that a fixed point $\z^\star$ is a \textbf{strict linearly stable point} if $\rho(\J(\z^\star)) < 1$ and a \textbf{strict linearly unstable point} if $\rho(\J(\z^\star)) > 1$.
\end{definition}
Intuitively, linear stability captures the idea that under the dynamics $\z_{t+1} = \w(\z_t)$ a flow that starts at point that is infinitesimally close to $\z^\star$ will remain in a small neighborhood around $\z^\star$.

%% file: result_minimax.tex

\section{Main Results}
\label{sec:results}
In the previous section, we pointed out that while many modern applications are in fact sequential games, the problem of finding their optima---global minimax points---is NP-hard in general.
We now turn to our main results, which provide ways to circumvent this NP-hardness challenge. In Section~\ref{sec:minimax}, we develop a formal notion of local surrogacy for global minimax points which we refer to as \emph{local minimax points}.  In Section~\ref{sec:prop} we study their properties and existence. Finally, in Section~\ref{sec:gda}, we establish a close relationship between stable fixed points of GDA and local minimax points.

\subsection{Local minimax points}
\label{sec:minimax}

While most previous work~\citep{daskalakis2018limit,mazumdar2018convergence} has focused on local Nash equilibria (Definition~\ref{def:localNash}), which are local surrogates for pure strategy Nash equilibria for simultaneous games, we propose a new notion---\emph{local minimax}---as a natural local surrogate for global minimaxity (Definition \ref{def:globalminimax}) for sequential games. To the best of our knowledge, this is the first proper mathematical definition of local optimality for the two-player sequential setting.


\begin{figure}[t]
    \centering
    \includegraphics[width=0.45\textwidth, trim=0 0 0 0, clip]{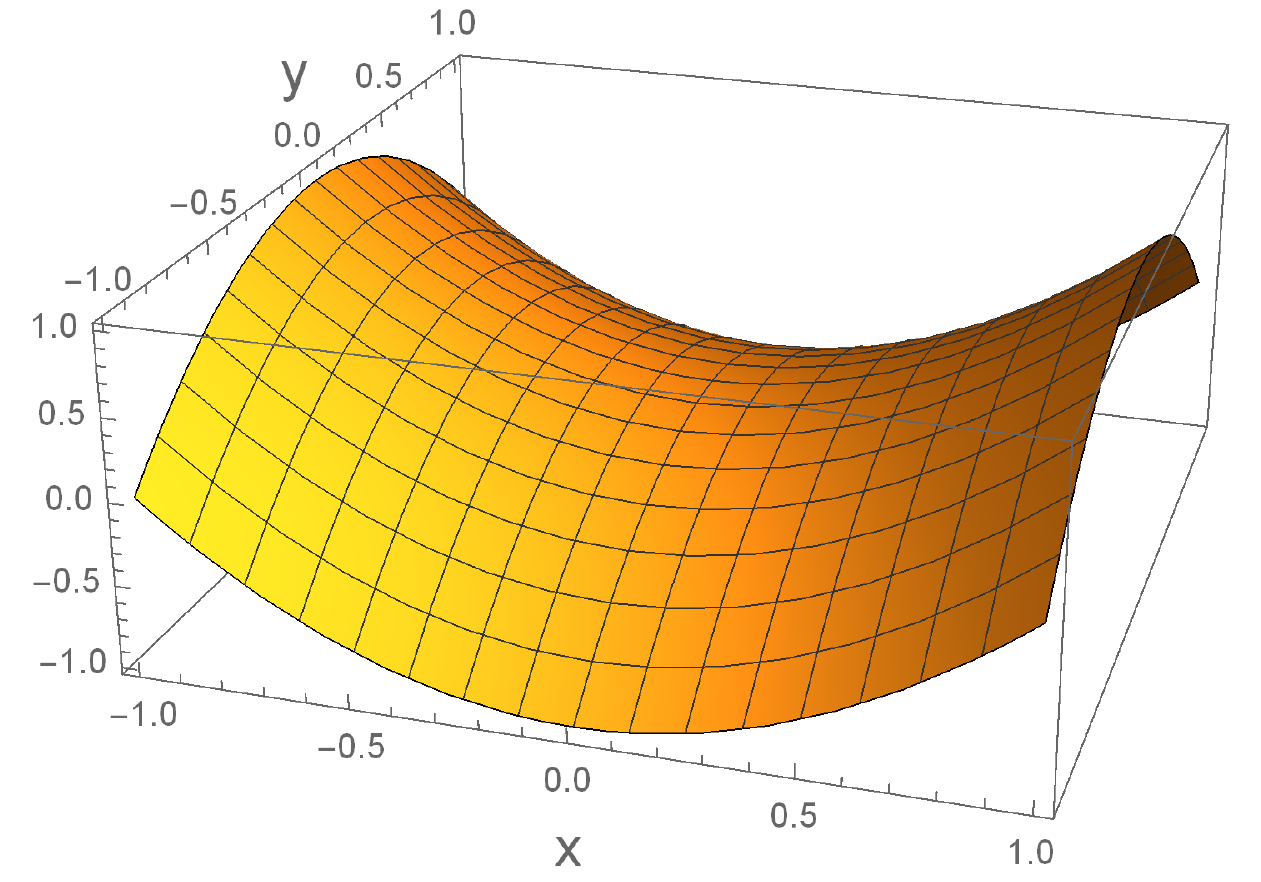}
    \includegraphics[width=0.45\textwidth, trim=0 0 0 0, clip]{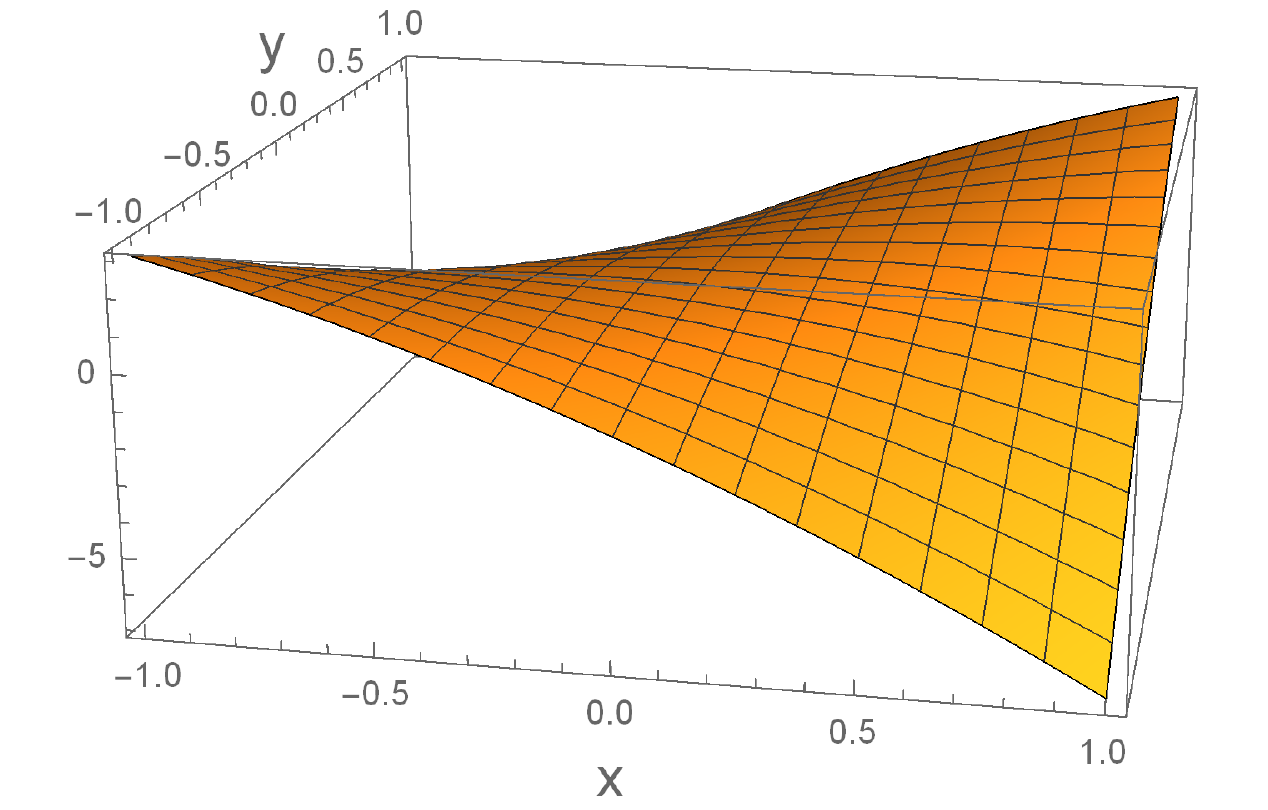}
    \caption{\textbf{Left:} $f(x, y) = x^2 - y^2$ where $(0,0)$ is both local Nash and local minimax. \textbf{Right:} $f(x, y) = -x^2 + 5xy - y^2$ where $(0,0)$ is not local Nash but local minimax with $h(\delta)=\delta$.}
    \label{fig:NashvsMinimax}
\end{figure}

\begin{restatable}{definition}{DEFlocal}\label{def:localminimax}
A point $(\x^\star, \y^\star)$ is said to be a \textbf{local \minimax} of $f$, if there exists $\delta_0>0$ and a function $h$ satisfying $h(\delta) \rightarrow 0$ as $\delta\rightarrow 0$, such that for any $\delta \in (0, \delta_0]$, and any $(\x, \y)$ satisfying $\norm{\x - \x^\star} \le \delta$ and $\norm{\y - \y^\star} \le \delta$, we have
\begin{equation}\label{eq:deflocalminimax}
f(\x^\star, \y) \le f(\x^\star, \y^\star) \le \max_{\y': \norm{\y' - \y^\star} \le h(\delta)} f(\x, \y').
\end{equation}
\end{restatable}

\begin{restatable}{remark}{RMdef}
Definition \ref{def:localminimax} remains equivalent even if we further restrict function $h$ in Definition \ref{def:localminimax} to be monotonic or continuous. See Appendix \ref{sec:localminimax} for more details.
\end{restatable}

Intuitively, local minimaxity captures the optimal strategies in a two-player sequential game if both players are only allowed to change their strategies locally.

Definition \ref{def:localminimax} localize the notion of global minimax points (Definition \ref{def:globalminimax}) by replacing all global optimality over $\x$ and $\y$ by local optimality. However, since this is a sequential setting, the radius of the local neighborhoods where the maximization or minimization takes over can be different. Definition \ref{def:localminimax} allows one radius to be $\delta$ while the other is $h(\delta)$. The introduction of an arbitrary function $h$ allows the ratio of these two radii to also be arbitrary. The limiting behavior $h(\delta) \rightarrow 0$ as $\delta\rightarrow 0$ makes this definition a truly local notion. That is, it only depends on the property of function $f$ in an infinitesimal neighborhood around $(\x^\star, \y^\star)$.

Definition \ref{def:localminimax} is a natural local surrogate for global minimax points. We can alternatively define local minimax points as localized versions of the equivalent characterization of global minimax points as in Remark \ref{rm:global}. It turns out that two definitions are equivalent.

\begin{restatable}{lemma}{LEMeqdef}\label{lem:eqdef}
For a continuous function $f$, a point $(\x^\star, \y^\star)$ is a local \minimax of $f$ if and only if $\y^\star$ is a local maximum of function $f(\cdot, \x^\star)$, and there exists an $\epsilon_0 > 0$ such that $\x^\star$ is a local minimum of function $g_{\epsilon}$ for all $\epsilon \in (0, \epsilon_0]$ where function $g_{\epsilon}$ is defined as $g_{\epsilon}(\x) \defeq \max_{\y:\norm{\y -\y^\star}\le \epsilon}f(\x, \y)$.
\end{restatable}

Lemma \ref{lem:eqdef} states that local minimaxity can be viewed from a game-theoretic perspective: the second player always plays the action to achieve a local maximum value $g_{\epsilon}(\x) \defeq \max_{\y:\norm{\y-\y^\star} \le \epsilon} f(\x, \y)$, for infinitesimal $\epsilon$, given the action $\x$ taken the first player, and the first player minimizes $g_{\epsilon}(\x)$ locally.

Finally, it can be shown that local minimaxity is a weakening of the notion of local Nash equilibrium defined as in Definition \ref{def:localNash}.  It alleviates the non-existence issues of the latter.


\begin{restatable}{proposition}{PROPnashisminimax}
Any local Nash equilibrium (Definition \ref{def:localNash}) is a local \minimax.
\end{restatable}

\subsection{Properties and existence}\label{sec:prop}


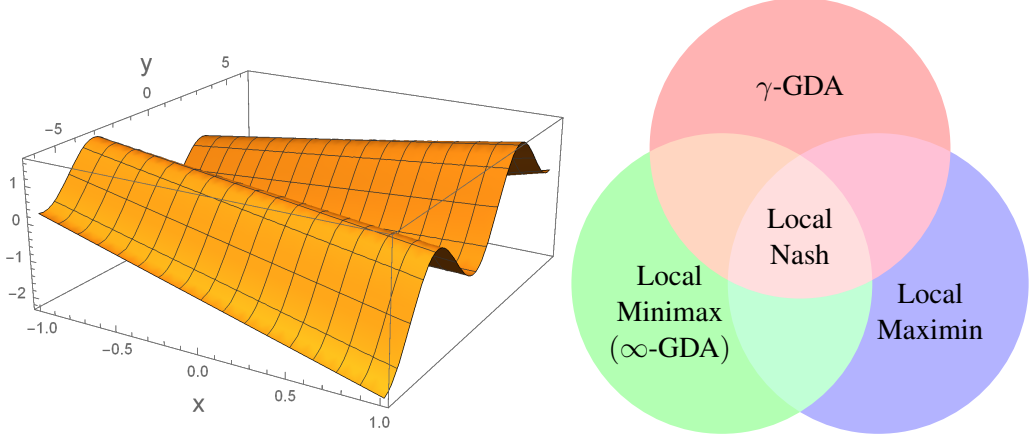
\begin{figure}[t]
    \centering
    \includegraphics[width=0.45\textwidth, trim=0 0 0 0, clip]{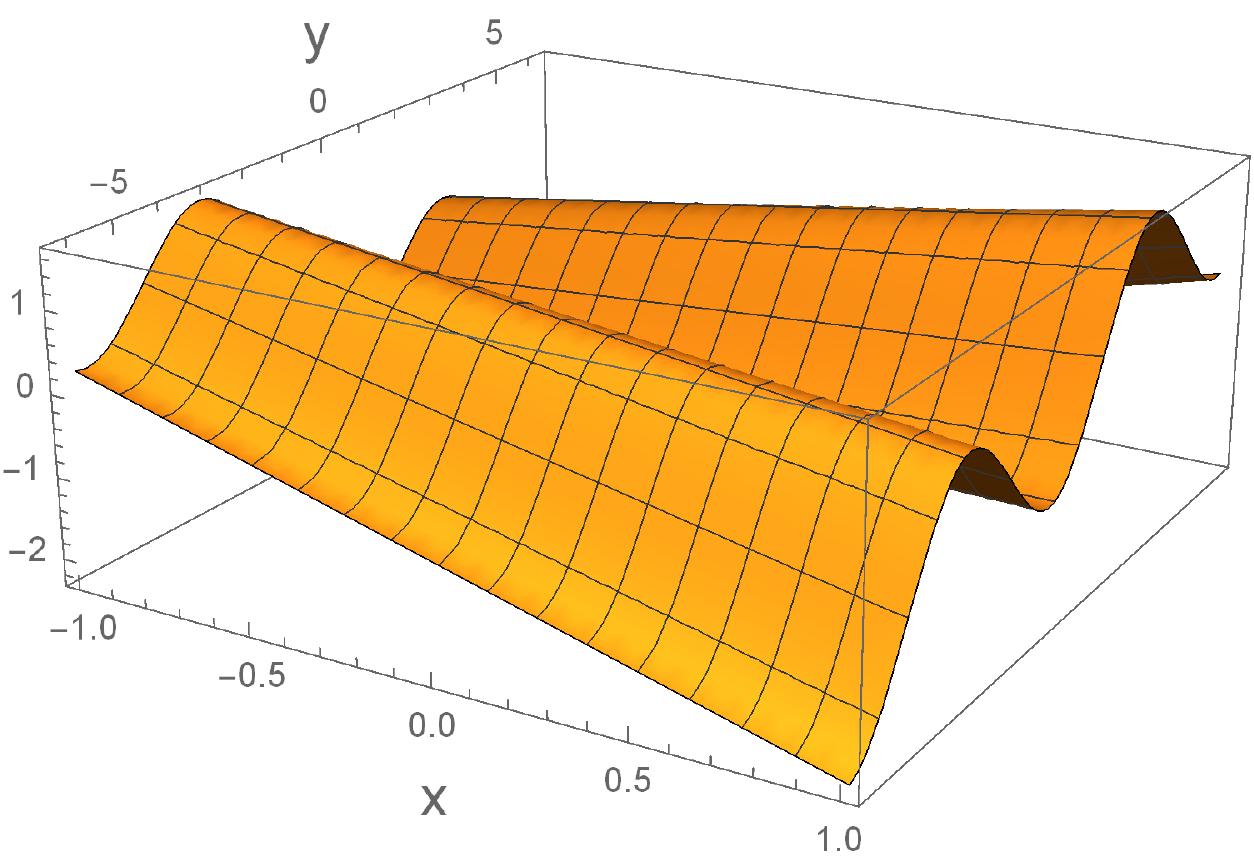}
    \begin{tikzpicture}
  \begin{scope}[blend group = soft light]
    \fill[red!30!white]   ( 90:1.2) circle (2);
    \fill[green!30!white] (210:1.2) circle (2);
    \fill[blue!30!white]  (330:1.2) circle (2);
  \end{scope}
  \node at ( 90:2)    {$\gamma$-GDA};
  \node at ( 210:2)   {$\begin{array}{c} \text{Local}\\ \text{Minimax}\\ (\infty\text{-GDA}) \end{array}$};
  \node at ( 330:2)   {$\begin{array}{c} \text{Local}\\ \text{Maximin} \end{array}$};
  \node  {$\begin{array}{c} \text{Local}\\ \text{Nash} \end{array}$};
\end{tikzpicture}
    \caption{\textbf{Left:} $f(x, y) = 0.2xy -cos(y)$, the global minimax points $(0, -\pi)$ and $(0, \pi)$ are not stationary. \textbf{Right:} The relations among local Nash equilibria, local minimax points, local maximin points and linearly stable points of $\gamma$-GDA, and $\infty$-GDA (up to degenerate points).}
    \label{fig:surr}
\end{figure}

Local minimax points also enjoy simple first-order and second-order characterizations. Notably, the second-order conditions naturally reflect the order of the sequential game (who plays first).

\begin{restatable}[First-order Necessary Condition]{proposition}{PROPfirstminimax}
\label{prop:first_nece}
Assuming that $f$ is continuously differentiable, then any local \minimax $(\x, \y)$ satisfies $\grad_{\x} f(\x, \y) = 0$ and $\grad_{\y} f(\x, \y) = 0$.
\end{restatable}

\begin{restatable}[Second-order Necessary Condition]{proposition}{PROPsecondminimax}
\label{prop:second_nece}
Assuming that $f$ is twice differentiable, then $(\x, \y)$ is a local \minimax implies that
$\phess_{\y\y} f(\x, \y) \preceq \zero$. Furthermore, if $\phess_{\y\y} f(\x, \y) \prec \zero$, then
\begin{equation}\label{eq:necessary_minimax}
[\phess_{\x\x} f - \phess_{\x\y} f (\phess_{\y\y} f)^{-1} \phess_{\y\x} f](\x, \y) \succeq \zero.
\end{equation}
\end{restatable}

\begin{restatable}[Second-order Sufficient Condition]{proposition}{PROPstrictminimax}
\label{prop:first_suff}
Assume that $f$ is twice differentiable. Any stationary point $(\x,\y)$ satisfying $\phess_{\y\y} f(\x, \y) \prec \zero$ and
\begin{equation}\label{eq:strict_minimax}
[\phess_{\x\x} f - \phess_{\x\y} f (\phess_{\y\y} f)^{-1} \phess_{\y\x} f](\x, \y) \succ \zero
\end{equation}
is a local \minimax. We call stationary points satisfying~\eqref{eq:strict_minimax} \emph{strict local \minimaxs}.
\end{restatable}

We note that if $\phess_{\y\y} f(\x, \y)$ is non-degenerate, then the second-order necessary condition (Proposition \ref{prop:second_nece}) becomes $\phess_{\y\y} f(\x, \y) \prec \zero$ and $[\phess_{\x\x} f - \phess_{\x\y} f (\phess_{\y\y} f)^{-1} \phess_{\y\x} f](\x, \y) \succeq \zero$, which is identical to the sufficient condition in Eq.~\eqref{eq:strict_minimax} up to an equals sign. 

Comparing Eq.~\eqref{eq:strict_minimax} to the second-order sufficient condition for local Nash equilibrium in Eq.~\eqref{eq:strict_Nash}, we see 
that, instead of requiring $\phess_{\x\x} f(\x, \y)$ to be positive definite, local minimaxity requires the Shur complement to be positive definite. Contrary to local Nash equilibria, this characterization of local minimaxity not only takes into account the interaction term $\phess_{\x\y}f$ between $\x$ and $\y$, but also reflects the difference between the first player and the second player.

For existence, we would like to first highlight an interesting fact: in contrast to the well-known fact in nonconvex optimization that global minima are always local minima (thus local minima always exist), global minimax points can be neither local minimax nor even stationary points.

\begin{restatable}{proposition}{PROPnosurrogate}\label{prop:surr}
The global minimax point can be neither local minimax nor a stationary point. 
\end{restatable}

See Figure \ref{fig:surr} for an illustration and Appendix \ref{sec:localminimax} for the proof. The proposition is a natural consequence of the definitions where global minimax points are obtained as a minimum of a \emph{global} maximum function while local minimax points are the minimum of a \emph{local} maximum function. 
This also illustrates that minimax optimization is a challenging task, and worthy of independent study, beyond nonconvex optimization.

Therefore, although global minimax points always exist as in Proposition \ref{prop:minimaxexist}, it is not necessary for local minimax points to always exist. Unfortunately, similar to local Nash equilibria, local minimax points may not exist in general.
\begin{restatable}{lemma}{LEMlocalMMnoexist}
There exists a twice-differentiable function $f$ and a compact domain, where local minimax points do not exist.
\end{restatable}


Nevertheless, global minimax points can be guaranteed to be local minimax under some further regularity. For instance, this is true when $f$ is strongly-concave in $\y$, or more generally when $f$ satisfies the following properties that have been established in several machine learning problems~\citep{ge2017no,boumal2016non}. There, local minimax points are guaranteed to exist.


\begin{restatable}{theorem}{THMgoodsurr}
Assume that $f$ is twice differentiable, and for any fixed $\x$, the function $f(\x, \cdot)$ 
is strongly concave in the neighborhood of local maxima and  satisfies the assumption that all local maxima are global maxima.
Then the global minimax point of $f(\cdot,\cdot)$ is also a local minimax point.
\end{restatable}

%% file: result_gda.tex

\subsection{Relation to the limit points of GDA}
\label{sec:gda}

In this section, we consider the asymptotic behavior of Gradient Descent Ascent (GDA), and its relation to local minimax points. As shown in the pseudo-code in Algorithm \ref{algo:GDA}, GDA simultaneously performs gradient descent on $\x$ and gradient ascent on $\y$. We consider the general form where the step size for $\x$ can be different from the step size for $\y$ by a ratio $\gamma$, and denoted this algorithm by $\gamma$-GDA. When the step size $\eta$ is small, this is essentially equivalent to the algorithm that alternates between one step of gradient descent and $\gamma$ steps of gradient ascent. 

\begin{algorithm}[t]
\caption{Gradient Descent Ascent ($\gamma$-GDA)}\label{algo:GDA}
\begin{algorithmic}
\renewcommand{\algorithmicrequire}{\textbf{Input: }}
\renewcommand{\algorithmicensure}{\textbf{Output: }}
\REQUIRE $(\x_0, \y_0)$, step size $\eta$, ratio $\gamma$.
\FOR{$t = 0, 1, \ldots, $}
\STATE $\x_{t+1} \leftarrow \x_t - (\eta/\gamma) \pgrad_\x f(\x_t, \y_t)$.
\STATE $\y_{t+1} \leftarrow \y_t + \eta \pgrad_\y f(\x_t, \y_t)$.
\ENDFOR
\end{algorithmic}
\end{algorithm}

To study the limiting behavior, we primarily focus on linearly stable points of $\gamma$-GDA, since with random initialization, $\gamma$-GDA will almost surely escape strict linearly unstable points.


\begin{theorem}[\cite{daskalakis2018limit}]
For any $\gamma>1$, assuming the function $f$ is $\ell$-gradient Lipschitz, and the step size $\eta \le 1/\ell$, then the set of initial points $\x_0$ so that $\gamma$-GDA converges to a strict linear unstable point is of Lebesgue measure zero.
\end{theorem}

We further simplifiy the problem by considering the limiting case where $\eta \rightarrow 0$, which corresponds to $\gamma$-GDA flow.
We note the asymptotic behavior of $\gamma$-GDA flow is essentially the same as $\gamma$-GDA with a very small step size $\eta$ up to certain error tolerance.
\begin{equation*}
 \frac{\dd \x}{\dd t} = -\frac{1}{\gamma} \grad_\x f(\x, \y) \qquad 
 \frac{\dd \y}{\dd t} = \grad_\y f(\x, \y).
\end{equation*} 
The strict linearly stable points of the $\gamma$-GDA flow have a very simple second-order characterization.
\begin{restatable}{proposition}{PROPstable}
Point $(\x, \y)$ is a strict linearly stable point of $\gamma$-GDA if and only if for all the eigenvalues $\{\lambda_i\}$ of following Jacobian matrix, 
\begin{equation*}
\J_\gamma = \begin{pmatrix}
-(1/\gamma)\hess_{\x\x}f(\x, \y) & -(1/\gamma)\hess_{\x\y}f(\x, \y)\\
\hess_{\y\x}f(\x, \y) & \hess_{\y\y}f(\x, \y),
\end{pmatrix}
\end{equation*}
their real part $\Re(\lambda_i) <0$ for any $i$.
\end{restatable}

In the remainder of this section, we assume that $f$ is a twice-differentiable function,
and we use $\ttt{Local\_Nash}$ to represent the set of strict local Nash equilibria, $\ttt{Local\_Minimax}$ for the set of strict local minimax points, $\ttt{Local\_Maximin}$ for the set of strict local maximin points, and $\gamma\ttt{-GDA}$ for the set of strict linearly stable points of the $\gamma$-GDA flow. Our goal is to understand the relationship between these sets.
\citet{daskalakis2018limit} and \citet{mazumdar2018convergence} provided a relation between $\ttt{Local\_Nash}$ and $1\ttt{-GDA}$ which can be generalized to $\gamma\ttt{-GDA}$ as follows.
\begin{restatable}[\cite{daskalakis2018limit}]{proposition}{PROPlimitnash}
\label{prop:limit_nash}
For any fixed $\gamma$, for any twice-differentiable $f$, $\ett{Local\_Nash} \subset \gamma\ett{-GDA}$, but there exist twice-differentiable $f$ such that $\gamma\ett{-GDA} \not\subset \ett{Local\_Nash}$.
\end{restatable}

That is, if $\gamma$-GDA converges, it may converge to points not in $\ttt{Local\_Nash}$. This raises a basic question as to what those additional stable limit points of $\gamma$-GDA are. Are they meaningful? 
This paper answers this question through the lens
of $\ttt{Local\_Minimax}$.
Although for fixed $\gamma$, the set $\gamma\ttt{-GDA}$ does not have a simple relation with $\ttt{Local\_Minimax}$, it turns out that an important relationship arises when $\gamma$ goes to $\infty$. 
To describe the limit behavior of the set $\gamma\ttt{-GDA}$ when $\gamma \rightarrow \infty$ we define two set-theoretic limits:
\begin{align*}
\overline{\infty\ttt{-GDA}} \defeq & \limsup_{\gamma\rightarrow \infty} \gamma\ttt{-GDA} =  \cap_{\gamma_0>0} \cup_{\gamma>\gamma_0} \gamma\ttt{-GDA}\\
\underline{\infty\ttt{-GDA}} \defeq &\liminf_{\gamma\rightarrow \infty}\gamma\ttt{-GDA} = \cup_{\gamma_0>0} \cap_{\gamma>\gamma_0} \gamma\ttt{-GDA}.
\end{align*}
The relations between $\gamma\ttt{-GDA}$ and $\ttt{Local\_Minimax}$ are given as follows:

\begin{restatable}{proposition}{PROPlimitminimax}
\label{prop:limit_minimax}
For any fixed $\gamma$,  there exists a twice-differentiable $f$ such that $\ett{Local\_Minimax} \not \subset \gamma\ett{-GDA}$; there also exists a twice-differentiable $f$ such that $\gamma\ett{-GDA} \not \subset \ett{Local\_Minimax} \cup \ett{Local\_Maximin}$.
\end{restatable}

\begin{restatable}[Asymptotic Behavior of $\infty$-GDA]{theorem}{THMgdamain}
\label{thm:main} For any twice-differentiable $f$, 
$\ett{Local\_Minimax} \subset \underline{\infty\ttt{-GDA}} \subset \overline{\infty\ttt{-GDA}} \subset \ett{Local\_Minimax} \cup \{(\x, \y)| (\x, \y)$ is stationary and $\phess_{\y\y} f(\x, \y)$ is degenerate$\}$.
\end{restatable}

That is, $\infty\ttt{-GDA} = \ttt{Local\_Minimax}$ up to some degenerate points. Intuitively, when $\gamma$ is large, $\gamma$-GDA can move a long distance in $\y$ while only making very small changes in $\x$. As $\gamma \rightarrow \infty$, $\gamma$-GDA can find the ``approximate local maximum'' of $f(\x + \dx, \cdot)$, subject to any small change in $\dx$; therefore, stable limit points are indeed local minimax. 

Algorithmically, one can view $\infty\ttt{-GDA}$ as a set that describes the strict linear stable limit points for GDA with $\gamma$ very slowly increasing with respect to $t$, and eventually going to $\infty$. To the best of our knowledge, this is the first result showing that all stable limit points of GDA are meaningful and locally optimal up to some degenerate points. 


%% file: result_oracle.tex

\section{Gradient Descent with Max-Oracle}
\label{sec:oracle}
In this section, we consider solving the minimax problem Eq.\eqref{eq:minimaxprob} when we have access to an oracle for approximate inner maximization; i.e., for any $\x$, we have access to an oracle that outputs a $\yhat$ such that $f(\x,\yhat) \geq \max_{\y} f(\x,\y) - \epsilon$. A natural algorithm to consider in this setting is to alternate between gradient descent on $\x$ and a (approximate) maximization step on $\y$. The pseudocode is presented in Algorithm~\ref{algo:GDMO}.
\begin{algorithm}[t]
\caption{Gradient Descent with Max-oracle}\label{algo:GDMO}
\begin{algorithmic}
\renewcommand{\algorithmicrequire}{\textbf{Input: }}
\renewcommand{\algorithmicensure}{\textbf{Output: }}
\REQUIRE $\x_0$, step size $\eta$.
\FOR{$t = 0, 1, \ldots, T$}
\STATE find $\y_t$ so that $f(\x_t, \y_t) \ge \max_{\y} f(\x_t, \y) - \epsilon$.
\STATE $\x_{t+1} \leftarrow \x_t - \eta \pgrad_\x f(\x_t, \y_t)$.
\ENDFOR
\STATE Pick $t$ uniformly at random from $\{0,\cdots,T\}$.
\STATE \textbf{return} $\bar{\x} \leftarrow \x_t$.
\end{algorithmic}
\end{algorithm}

It can be shown that Algorithm~\ref{algo:GDMO} indeed converges (in contrast with GDA which can converge to limit cycles). Moreover, the limit points of Algorithm~\ref{algo:GDMO} satisfy a nice property---they turn out to be approximately stationary points of $\phi(\x) \defeq \max_{\y} f(\x, \y) $. For a smooth function, ``approximately stationary point'' means that the norm of gradient is small. However, even when $f(\cdot,\cdot)$ is smooth (up to whatever order), $\phi(\cdot)$ as defined above need not be differentiable.
The norm of the subgradient can be a discontinuous function which is undesirable for a measure of closeness to stationarity. Fortunately, however, we have following structure.

\begin{fact} 
If function $f: \cX \times \cY \rightarrow \R$ is $\ell$-gradient Lipschitz, then function $\phi(\cdot) \defeq \max_{\y \in \cY} f(\cdot, \y)$ is $\ell$-\textbf{weakly convex}~; that is, $\phi(\x) + (\ell/2) \norm{\x}^2$ is a convex function over $\x$.
\end{fact}  
The above fact has been also presented in \citet{rafique2018non}.
In such settings, the approximate stationarity of $\phi(\cdot)$ can be measured by the norm of the gradient of the Moreau envelope $\phi_{\lambda}(\cdot)$.
\begin{align}
	\phi_{\lambda}(\x) \defeq \min_{\x'} \phi(\x') + \frac{1}{2\lambda} \norm{\x - \x'}^2.\label{eq:moreau}
\end{align}
Here $\lambda < 1/\ell$ is the parameter. 
The Moreau envelope has the following important property that connects it to the original function $\phi$.
\begin{lemma}[\citep{rockafellar2015convex}]\label{lem:meaning_Moreau}
Assume the function $\phi$ is $\ell$-weakly convex. Let $\lambda<1/\ell$, and denote $\hat{\x} = \argmin_{\x'} \phi(\x') + (1/2\lambda) \norm{\x - \x'}^2$. Then $\norm{\grad \phi_{\lambda}(\x)} \le \epsilon$ implies:
\begin{equation*}
\norm{\hat{\x} - \x} = \lambda \epsilon,\quad \text{~and~}\quad \min_{\g \in \partial \phi(\hat{\x})}\norm{\g} \le \epsilon,
\end{equation*}
where $\partial$ denotes the subdifferential of a weakly convex function. 
\end{lemma}

Lemma \ref{lem:meaning_Moreau} says that $\norm{\grad \phi_{\lambda}(\x) }$ being small means that $\x$ is close to a point $\hat{\x}$ that is an approximately stationary point of original function $\phi$.
We now present the convergence guarantee for Algorithm~\ref{algo:GDMO}.

\begin{restatable}{theorem}{THMmaxoracle}
\label{thm:max_oracle}
Suppose $f$ is $\ell$-smooth and $L$-Lipschitz and define $\phi(\cdot) \defeq \max_{\y} f(\cdot, \y) $. Then the output $\bar{\x}$ of GD with Max-Oracle (Algorithm~\ref{algo:GDMO}) with step size $\eta = \gamma/\sqrt{T+1}$ satisfies
\begin{align*}
	\E&\left[\norm{\nabla \phiell(\bar{\x})}^2\right] \leq 2 \cdot \frac{\left(\phiell(\x_0) - \min \phi(\x)\right) + \ell L^2 \gamma^2}{\gamma \sqrt{T+1}} + 4 \ell \epsilon,
\end{align*}
where $\phiell$ is the Moreau envelope~\eqref{eq:moreau} of $\phi$.
\end{restatable}
The proof of Theorem \ref{thm:max_oracle} is similar to the convergence analysis for nonsmooth weakly-convex functions \citep{davis2018stochastic}, except here the max-oracle has error $\epsilon$.
Theorem \ref{thm:max_oracle} claims, other than an additive error $4 \ell \epsilon$ as a result of the oracle solving the maximum approximately, that the remaining term decreases at a rate of $1/\sqrt{T}$. 

%% file: conclu.tex

\section{Conclusions}\label{sec:conc}
In this paper, we consider general nonconvex-nonconcave minimax optimization problems. Since most these problems arising in modern machine learning correspond to sequential games, we propose a new notion of local optimality---\emph{local minimax}---the first proper mathematical definition of local optimality for the two-player sequential setting. We present favorable results on their properties and existence. We also establish a strong connection to GDA---up 
to some degenerate points, local minimax points are exactly equal to the stable limit points of GDA. 

%% file: mixed.tex

\section{Relation to Mixed Strategy Games}
\label{sec:mixed}


In contrast to pure strategies where each player plays a single action, game theorists have also considered mixed strategies where each player is allowed to play a randomized action sampled from a probability measure $\mu \in \cP(\cX)$ or $\nu \in \cP(\cY)$. Then, the payoff function becomes an expected value $\E_{\x\sim \mu, \y\sim \nu} f(\x, \y)$. For mixed strategy games, even if function is nonconvex-nonconcave, the following minimax theorem still holds.
\begin{proposition}[\citep{glicksberg1952further}]\label{prop:mixedNashexists}
Assume that the function $f: \cX \times \cY \rightarrow \R$ is continuous and that $\cX \subset \R^{d_1}$, $\cY \subset \R^{d_2}$ are compact. Then
\begin{equation*}
\min_{\mu \in \cP(\cX)}\max_{\nu \in \cP(\cY)}\E_{(\mu, \nu)} f(\x, \y)
=\max_{\nu \in \cP(\cY)}\min_{\mu \in \cP(\cX)}\E_{(\mu, \nu)} f(\x, \y).
\end{equation*}
\end{proposition}
This implies that the order which player goes first is no longer important in this setting, and there is no intrinsic difference between simultaneous games and sequential games if mixed strategies are allowed.

Similar to the concept of (pure strategy) Nash equilibrium in Definition \ref{def:pureNash}, we can define mixed strategy Nash equilibria as 
\begin{definition}\label{def:mixedNash}
A probability measure $(\mu^\star, \nu^\star)$ is a \textbf{mixed strategy Nash equilibrium} of $f$, if for any measure $(\mu, \nu)$ in $\cP(\cX) \times \cP(\cY)$, we have
$$\E_{\x \sim \mu^\star, \y \sim \nu} f(\x, \y) \le \E_{\x \sim \mu^\star, \y \sim \nu^\star} f(\x, \y) \le \E_{\x \sim \mu, \y \sim \nu^\star} f(\x, \y).$$
\end{definition}
Unlike pure strategy Nash equilibrium, the existence of mixed strategy Nash equilibrium in this setting is always guaranteed \cite{glicksberg1952further}.

One challenge for finding mixed strategy equilibria is that it requires optimizing over a space of probability measures, which is of infinite dimension. However, we can show that finding approximate mixed strategy Nash equilibria for Lipschitz games can be reduced to finding a global solution of a ``augmented'' pure strategy sequential games, which is a problem of polynomially large dimension.
\begin{definition}Let $(\mu^\star, \nu^\star)$ be a mixed strategy Nash equilibrium.
A probability measure $(\mu^\dagger, \nu^\dagger)$ is an \textbf{$\epsilon$-approximate mixed strategy Nash equilibrium} if:
\begin{align*}
\forall \nu' \in \cP(\cY),\quad &\E_{(\mu^\dagger, \nu')} f(\x, \y) \le \E_{(\mu^\star, \nu^\star)} f(\x, \y) + \epsilon\\
\forall \mu' \in \cP(\cY),\quad &\E_{(\mu', \nu^\dagger)} f(\x, \y) \ge \E_{(\mu^\star, \nu^\star)} f(\x, \y) - \epsilon.
\end{align*}
\end{definition}

\begin{restatable}{theorem}{THMreduction}
\label{thm:reduction}
Assume that function $f$ is $L$-Lipschitz, and the diameters of $\cX$ and $\cY$ are at most $D$. Let $(\mu^\star, \nu^\star)$ be a mixed strategy Nash equilibrium. Then there exists an absolute constant $c$, for any $\epsilon>0$, such that if $N \ge c\cdot d_2 (L D/\epsilon)^2 \log (LD/\epsilon)$, we have:
\begin{equation*}
\min_{(\x_1, \ldots, \x_N) \in \cX^N}\max_{\y \in \cY} ~ \frac{1}{N}\sum_{i=1}^N f(\x_i, \y) 
\le 
\E_{(\mu^\star, \nu^\star)} f(\x, \y)
 + \epsilon.
\end{equation*}
\end{restatable}
Intuitively, Theorem \ref{thm:reduction} holds because the function $f$ is Lipschitz, $\cY$ is a bounded domain, and thus we can establish uniform convergence of the expectation of $f(\cdot, \y)$ to its average over $N$ samples for all $\y \in \cY$ simultaneously. A similar argument was made in~\cite{arora2017generalization}.

Theorem \ref{thm:reduction} implies that in order to find an approximate mixed strategy Nash equilibrium, we can solve a large minimax problem with objective $F(\X, \y) \defeq \sum_{i=1}^N f(\x_i, \y)/N$. The global minimax solution $\X^\star = (\x_1^\star, \ldots, \x_n^\star)$ gives a empirical distribution $\hat{\mu}^\star = \sum_{i=1}^N \delta(\x - \x_i^\star)/N$, where $\delta(\cdot)$ is the Dirac delta function. 
By symmetry, we can also solve the corresponding maximin problem to find $\hat{\nu}^\star$. 
It can be shown that $(\hat{\mu}^\star, \hat{\nu}^\star)$ is an $\epsilon$-approximate mixed strategy Nash equilibrium. That is, approximate mixed strategy Nash can be found by finding two global minimax points.

\begin{proof}[Proof of Theorem \ref{thm:reduction}]
Note that without loss of generality, the second player can always play a pure strategy. That is,
\begin{equation*}
\min_{\mu \in \mathcal{P}(\cX)} \max_{\nu \in \mathcal{P}(\cY)}
\E_{\x \sim \mu, \y \sim \nu} f(\x, \y)
= \min_{\mu \in \mathcal{P}(\cX)} \max_{\y\in \cY}
\E_{\x \sim \mu} f(\x, \y).
\end{equation*}
Therefore, we only need to solve the problem on the right-hand side. Suppose the minimum over $\mathcal{P}(\cX)$ is achieved at $\mu^\star$. First, sample $(\x_1, \ldots, \x_N)$ i.i.d from $\mu^\star$, and note that $\max_{\x_1, \x_2 \in \cX} |f(\x_1, \y) - f(\x_2, \y)| \le LD$ for any fixed $\y$. Therefore by the Hoeffding inequality, for any fixed $\y$:
\begin{equation*}
\Pr\left(\frac{1}{N} \sum_{i=1}^N f(\x_i, \y)  - \E_{\x\sim \mu^\star} f(\x, \y)\ge t  \right) \le e^{-\frac{Nt^2}{(LD)^2}}.
\end{equation*}
Let $\bar{\cY}$ be a minimal $\epsilon/(2L)$-covering over $\cY$. We know the covering number $|\bar{\cY}| \le (2DL/\epsilon)^d$. Thus by a union bound:
\begin{equation*}
\Pr\left(\forall \y \in \bar{\cY}, ~\frac{1}{N} \sum_{i=1}^N f(\x_i, \y)  - \E_{\x\sim \mu^\star} f(\x, \y)\ge t  \right) \le e^{d \log \frac{2DL}{\epsilon} - \frac{Nt^2}{(LD)^2}}
\end{equation*}
Picking $t = \epsilon/2$ and letting $N \ge c\cdot d(L D/\epsilon)^2 \log(L D/\epsilon)$ for some large absolute constant $c$, we have:
\begin{equation*}
\Pr\left(\forall \y \in \bar{\cY}, ~\frac{1}{N} \sum_{i=1}^N f(\x_i, \y)  - \E_{\x\sim \mu^\star} f(\x, \y)\ge \frac{\epsilon}{2}  \right) \le \frac{1}{2}.
\end{equation*}
Letting $\y^\star =  \arg\max_\y \frac{1}{N} \sum_{i=1}^N f(\x_i, \y)$, by definition of covering, we can always find a $\y' \in \bar{\cY}$ so that $\norm{\y^\star - \y'} \le \epsilon/(4L)$. Thus, with probability at least $1/2$:
\begin{align*}
&\max_\y \frac{1}{N} \sum_{i=1}^N f(\x_i, \y)  - \max_\y \E_{\x\sim \mu^\star} f(\x, \y)
= \frac{1}{N} \sum_{i=1}^N f(\x_i, \y^\star)  - \max_\y \E_{\x\sim \mu^\star} f(\x, \y)\\
\le & \left[\frac{1}{N} \sum_{i=1}^N f(\x_i, \y^\star) - \frac{1}{N} \sum_{i=1}^N f(\x_i, \y')\right] + \left[\frac{1}{N} \sum_{i=1}^N f(\x_i, \y') - \E_{\x\sim \mu^\star} f(\x, \y')\right]\\
&+ [\E_{\x\sim \mu^\star} f(\x, \y') - \max_\y \E_{\x\sim \mu^\star} f(\x, \y)]
\le  \epsilon/2 + \epsilon/2 + 0 \le \epsilon
\end{align*}
That is, with probability at least $1/2$:
\begin{equation*}
\max_{\y\in \cY} \frac{1}{N} \sum_{i=1}^N f(\x_i, \y) \le \min_{\mu \in \mathcal{P}(\cX)} \max_{\y\in \cY}
\E_{\x \sim \mu} f(\x, \y) + \epsilon.
\end{equation*}
This implies:
\begin{equation*}
\min_{(\x_1, \ldots, \x_N) \in \cX^N}\max_{\y \in \cY} \frac{1}{N} \sum_{i=1}^N f(\x_i, \y) \le \min_{\mu \in \mathcal{P}(\cX)} \max_{\y\in \cY} \E_{\x \sim \mu} f(\x, \y) + \epsilon.
\end{equation*}
Combining with Proposition \ref{prop:mixedNashexists}, we finish the proof.
\end{proof}

%% file: relation.tex

\section{Relation to~\cite{evtushenko1974some}}
\label{sec:relation_earlier}


In this section, we review a notion similar to local minimax proposed by \citet{evtushenko1974some}. To distinguish that notion from our definition (Definition \ref{def:localminimax}), we call it Evtushenko's minimax property. We remark that Evtushenko's minimax is not a truly local property. As a noticable difference, Evtushenko's definition does not satisfy the first-order and second-order necessary conditions of the local minimax notion proposed in this paper as in Proposition \ref{prop:first_nece} and Proposition \ref{prop:second_nece}.

Evtushenko's definition of local minimax point can be stated as follows.
\begin{definition}[\citep{evtushenko1974some}] 
\label{def:Eslocal}
A point $(\x^\star, \y^\star)$ is said to be a \textbf{Evtushenko's \minimax} of $f$, if there exist a local neighborhood $\cW$ of $(\x^\star, \y^\star)$ so that $(\x^\star, \y^\star)$ is a global minimax point (Definition \ref{def:globalminimax}) within $\cW$.
\end{definition}

First, we remark that Definition \ref{def:Eslocal} is in fact not a local notion. That is, whether a point $(\x^\star, \y^\star)$ is a Evtushenko's \minimax relies on the property of function $f$ at points which are far away from $(\x^\star, \y^\star)$.

\begin{proposition}
There exists a twice differentiable function $f$, a point $(\x, \y)$ and its two local neighborhoods $\cW_1, \cW_2$ satisfying $\cW_1 \subset \cW_2$, so that $(\x, \y)$ \textbf{is not} a Evtushenko's \minimax for $f|_{\cW_1}$ but \textbf{is} a Evtushenko's \minimax for $f|_{\cW_2}$. Here $f|_\cW$ denotes the function $f$ restriced to the domain $\cW$.
\end{proposition}

\begin{proof}
Consider the function $f(x, y) = 0.2 xy - \cos(y)$ in the region $\cW_2 = [-1, 1] \times [-2\pi, 2\pi]$ as shown in Figure \ref{fig:surr}.
Follow the same analysis as in the proof of Proposition \ref{prop:surr}, we can show that $(0, -\pi)$ is a global minimax point on $\cW_2$, thus a Evtushenko's \minimax of $f|_{\cW_2}$.


On the other hand, consider $\cW_1 = [-1, 1] \times [-2\pi, 0]$. For any fixed $x$, the global maximium $y^\star(x)$ satisfies $0.2x + \sin(y^\star) = 0$ where $y^\star(x) \in (-3\pi/2, -\pi/2)$. Then for any local neighborhood $\cW$ of $(0, -\pi)$ such that $\cW \subset \cW_1$, there always exists an $\epsilon >0$ so that $(\epsilon, y^\star(\epsilon)) \in \cW$. However, we can verify that 
$$ f(0, -\pi) \ge  f(\epsilon, y^\star(\epsilon)) = \max_{y: (\epsilon, y) \in \cW} f(\epsilon, y)$$
That is $(0, -\pi)$ is not a Evtushenko's \minimax on $f|_{\cW_1}$. 

This concludes that whether $(0, -\pi)$ is a Evtushenko's \minimax depends on the property of function $f$ on set $\cW_2 - \cW_1$ whose elements are all far away from $(0, -\pi)$. That is, Evtushenko's minimax property is not a local property.
\end{proof}

We further clarify the relation between Evtushenko's \minimax and our definition of local minimax point (Definition \ref{def:localminimax}) as follows.
\begin{proposition} \label{prop:relation_EtoOurs}
A local minimax point is a Evtushenko's \minimax, but the reverse is not true.
\end{proposition}
\begin{proof}
If a point $(\x^\star, \y^\star)$ is a local minimax point, according to Definition \ref{def:localminimax}, there exists $\delta_0>0$ and a function $h$ satisfying $h(\delta) \rightarrow 0$ as $\delta\rightarrow 0$, such that for any $\delta \in (0, \delta_0]$, and any $(\x, \y)$ satisfying $\norm{\x - \x^\star} \le \delta$ and $\norm{\y - \y^\star} \le \delta$, we have
\begin{equation*}
f(\x^\star, \y) \le f(\x^\star, \y^\star) \le \max_{\y': \norm{\y' - \y^\star} \le h(\delta)} f(\x, \y').
\end{equation*}
Therefore, we can choose a $\delta^\dagger \in (0, \delta_0]$ so that $h(\delta^\dagger) \le \delta_0$. According to the equation above, $(\x^\star, \y^\star)$ is the global minimax point in region $\mathbb{B}_{\x^\star}(\delta^\dagger) \times \mathbb{B}_{\y^\star}(h(\delta^\dagger))$ where $\mathbb{B}_{\z}(r)$ denote the Euclidean ball around $\z$ with radius $r$. Therefore, $(\x^\star, \y^\star)$ is a Evtushenko's \minimax.

The claim that a Evtushenko's \minimax can be a non local minimax point easily follows from Proposition \ref{prop:surr}, where a global minimax point (which is always a Evtushenko's \minimax) can be a non local minimax point.
\end{proof}


Finally, as a consequence of Proposition \ref{prop:relation_EtoOurs}, the sufficient conditions for local minimax points are still sufficient conditions for Evtushenko's minimax points. However, the necessary conditions for local minimax points are no longer the necessary conditions for Evtushenko's minimax points.

\begin{proposition}
A Evtushenko's \minimax can be a non-stationary point, which does not satisfies Eq.\eqref{eq:necessary_minimax} even if $\phess_{\y\y} f(\x, \y) \prec \zero$.
\end{proposition}

\begin{proof}
Consider the function $f(x, y) = -0.03 x^2 + 0.2 xy - \cos(y)$ in the region $[-1, 1] \times [-2\pi, 2\pi]$ as shown in Figure \ref{fig:surr}.
Follow a similar analysis as in the proof of Proposition \ref{prop:surr}, we can show that $(0, -\pi)$ is a global minimax point, thus a Evtushenko's \minimax. However, we have gradient and Hessian at $(0, -\pi)$:
\begin{equation*}
\grad f = 
\begin{pmatrix}
-0.2\pi\\
0
\end{pmatrix}, \qquad 
\hess f = 
\begin{pmatrix}
-0.06 & 0.2\\
0.2 & -1
\end{pmatrix}
\end{equation*}
Therefore, $(0, -\pi)$ is not a stationary point, and despite $\phess_{\y\y} f(\x, \y) \prec \zero$, Eq.\eqref{eq:necessary_minimax} does not hold.
\end{proof}


%% file: pf_minimax.tex

\section{Proofs for Section \ref{sec:minimax}}
\label{sec:localminimax}
In this section, we prove the propositions and theorems presented in Section \ref{sec:minimax}.

\DEFlocal*

\RMdef*
\begin{proof}
Let $\cD_{gen}$ be the set of local minimax points according to Definition \ref{def:localminimax}. Let $\cD_{mon}, \cD_{cts}$ be the sets of points
if we further restrict function $h$ in Definition \ref{def:localminimax} to be monotonic or continuous. We will prove that
$\cD_{gen} \subset \cD_{mon} \subset \cD_{cts} \subset \cD_{gen}$, so that they are all equivalent.

For simplicity of presentation, we denote $\mathcal{P}(\delta, \epsilon)$ as the property that $f(\x^\star, \y^\star) \le g_{\epsilon}(\x)$ for any $\x$ satisfying $\norm{\x - \x^\star} \le \delta$. By its definition, we know if $\mathcal{P}(\delta, \epsilon)$ holds, then for any $\delta' \le \delta$ and any $\epsilon' \ge \epsilon$, $\mathcal{P}(\delta', \epsilon')$ also holds. 

\paragraph{$\cD_{gen} \subset \cD_{mon}$:} If $(\x^\star, \y^\star) \in \cD_{gen}$, we know there exists $\delta_0>0$ and function $h$ with $h(\delta)\rightarrow 0$ as $\delta \rightarrow 0$ such that for any $\delta \in (0, \delta_0]$, that $\mathcal{P}(\delta, h(\delta))$ holds. We can construct a function $h'(\delta) = \sup_{\delta \in (0, \delta]} h(\delta)$ for any $\delta \in (0, \delta_0]$. We can show $h'$ is a monotonically increasing function, and $h'(\delta) \rightarrow 0$ as $\delta \rightarrow 0$. Since $h'(\delta) \ge h(\delta)$ for all $\delta \in (0, \delta_0]$, we know $\mathcal{P}(\delta, h'(\delta))$ also holds, that is, $(\x^\star, \y^\star) \in \cD_{mon}$.

\paragraph{$\cD_{mon} \subset \cD_{cts}$:} For any monotonically increasing function $h: (0, \delta_0] \rightarrow \R$ with $h(\delta)\rightarrow 0$ as $\delta \rightarrow 0$, by a standard argument in analysis, we can show there exists a continuous function $h': (0, \delta_0] \rightarrow \R$ so that $h'(\delta) \ge h(\delta)$ for all $\delta \in (0, \delta_0]$ and $h'(\delta) \rightarrow 0$ as $\delta \rightarrow 0$.
Then, we can use similar arguments as above to finish the proof of this claim.

\paragraph{$\cD_{cts} \subset \cD_{gen}$:} This is immediate by definitions.
\end{proof}

\LEMeqdef*
\begin{proof}
For simplicity of presentation, we denote $\mathcal{P}(\delta, \epsilon)$ as the property that $f(\x^\star, \y^\star) \le g_{\epsilon}(\x)$ for any $\x$ satisfying $\norm{\x - \x^\star} \le \delta$. By its definition, we know if $\mathcal{P}(\delta, \epsilon)$ holds, then for any $\delta' \le \delta$ and any $\epsilon' \ge \epsilon$, $\mathcal{P}(\delta', \epsilon')$ also holds. Also, since $f$ is continuous, if a sequence $\{\epsilon_i\}$ has a limit, then $\mathcal{P}(\delta, \epsilon_i)$ holds for all $i$ implies that $\mathcal{P}(\delta, \lim_{i \rightarrow \infty}\epsilon_i)$ holds.

\paragraph{The ``only if'' direction:} Supposing $(\x^\star, \y^\star)$ is a local minimax point of $f$, then there exists $\delta_0>0$ and a function $h$ satisfying the properties stated in Definition \ref{def:localminimax}. Let $\epsilon_0 = \delta_0$. Since $h(\delta) \rightarrow 0$ as $\delta \rightarrow 0$, we know that for any $\epsilon \in (0, \epsilon_0]$, there exists $\delta \in (0, \delta_0]$ such that $h(\delta) \le \epsilon$. Meanwhile, according to Definition \ref{def:localminimax}, $\mathcal{P}(\delta, h(\delta))$ holds, which implies that $\mathcal{P}(\delta, \epsilon)$ holds. Finally, since $\epsilon \le \epsilon_0 = \delta_0$, we have $g_\epsilon(\x^\star) = f(\x^\star, \y^\star)$; i.e.,  $\y^\star$ achieves the local maximum. Combining this with the fact that $\mathcal{P}(\delta, \epsilon)$ holds, we finish the proof of this direction.

\paragraph{The ``if'' direction:} Since $\y^\star$ is a local maximum of $f(\cdot, \x^\star)$, there exists $\delta_y >0$ such that $g_{\delta_y}(\x^\star) = f(\x^\star, \y^\star)$. Let $\tilde{\epsilon}_0 = \min\{\epsilon_0, \delta_y\}$. By assumption, there exists a function $q$ such that for any $\epsilon \in (0, \tilde{\epsilon}_0]$ we have $q(\epsilon)>0$ and $\mathcal{P}(q(\epsilon), \epsilon)$ holds. Now, define $\delta_0 = q(\tilde{\epsilon}_0)>0$, and define a function $h$ on $(0, \delta_0]$ as follows:
\begin{equation*}
h(\delta) = 
\inf \{\epsilon | \epsilon \in (0, \tilde{\epsilon}] \text{~and~} q(\epsilon) \ge \delta\}.
\end{equation*}
It is easy to verify that when $\delta \in (0, \delta_0]$, the set on the RHS is always non-empty as $\tilde{\epsilon}_0$ is always an element of the set, thus the function $h$ is well-defined on its domain. First, it is clear that $h$ is a monotonically increasing function. Second, we prove $h \rightarrow 0$ as $\delta \rightarrow 0$, this is because for any $\epsilon'\in(0, \tilde{\epsilon}]$, there is a $\delta' = q(\epsilon')$ so that for any $\delta \in (0, \delta']$ we have $h(\delta) \le h(\delta') \le \epsilon'$. Finally, we note for any $\delta \in (0, \delta_0]$, by definition of $h(\delta)$, there exists a sequence $\{\epsilon_i\}$ so that $\lim_{i \rightarrow \infty}\epsilon_i = h(\delta)$, and for any $i$, we have $\epsilon_i \in [h(\delta), \tilde{\epsilon}_0]$ and $q(\epsilon_i) \ge \delta$. By assumption, we know $\mathcal{P}(q(\epsilon_i), \epsilon_i)$ holds, thus $\mathcal{P}(\delta, \epsilon_i)$ holds, which eventually implies $\mathcal{P}(\delta, h(\delta))$ since $\lim_{i \rightarrow \infty}\epsilon_i = h(\delta)$. This finishes the proof.
\end{proof}


\PROPnashisminimax*

\begin{proof}
Let $h$ be the constant function $h(\delta) = 0$ for any $\delta$.
If $(\x^\star, \y^\star)$ is a local \psne, then by definition this implies the existence of $\delta_0$ such that for any $\delta \le \delta_0$, and any $(\x, \y)$ satisfying $\norm{\x - \x^\star} \le \delta$ and $\norm{\y - \y^\star} \le \delta$: 
\begin{equation*}
f_2(\x^\star, \y) \le f_2(\x^\star, \y^\star) \le f(\x, \y^\star) \le \max_{\y': \norm{\y' - \y^\star} \le h(\delta)} f_2(\x, \y').
\end{equation*}
\end{proof}

\section{Proofs for Section \ref{sec:prop}}
In this section, we present proofs of the propositions and theorems presented in Section \ref{sec:prop}.

\PROPfirstminimax*

\begin{proof}
Since $\y$ is the local maximum of $f(\x, \cdot)$, we have $\grad_{\y} f(\x, \y) = 0$. Denote local optima $\dy^\star (\dx) \defeq$ $\argmax_{\norm{\dy} \le h(\delta)} f(\x + \dx, \y+\dy)$. By definition we know that $\norm{\dy^\star(\dx)} \le h(\delta) \rightarrow 0$ as $\delta \rightarrow 0$. Thus
\begin{align*}
0 & \le f(\x + \dx, \y + \dy^\star(\dx))  - f(\x, \y) \\
&= f(\x + \dx, \y + \dy^\star(\dx)) - f(\x, \y + \dy^\star(\dx))
+ f(\x, \y + \dy^\star(\dx)) - f(\x, \y)\\
& \le f(\x + \dx, \y + \dy^\star(\dx)) - f(\x, \y + \dy^\star(\dx))\\
& = \grad_\x f(\x, \y + \dy^\star(\dx))\trans \dx + o(\norm{\dx}) \\
&= \grad_\x f(\x, \y)\trans \dx + o(\norm{\dx})
\end{align*} 
holds for any small $\dx$, which implies $\grad_\x f(\x, \y) = 0$.
\end{proof}

\PROPsecondminimax*

\begin{proof}
Denote $\A \defeq \phess_{\x\x} f(\x, \y)$, $\B \defeq \phess_{\y\y} f(\x, \y)$ and $\C \defeq \phess_{\x\y} f(\x, \y)$.
Since $\y$ is the local maximum of $f(\x, \cdot)$, we have $\B \preceq 0$. On the other hand,
\begin{equation*}
f(\x+\dx, \y+\dy) = f(\x, \y) + \frac{1}{2}\dx\trans \A \dx + \dx\trans \C \dy + 
\frac{1}{2}\dy\trans \B \dy + o(\norm{\dx}^2 + \norm{\dy}^2).
\end{equation*}
Since $(\x, \y)$ is a local minimax point, by definition there exists a function $h$ such that Eq.~\eqref{eq:deflocalminimax} holds.
Denote $h'(\delta) = 2 \norm{\B^{-1}\C^\top} \delta$.
We note both $h(\delta)$ and $h'(\delta) \rightarrow 0$ as $\delta \rightarrow 0$. 
In case that $\B \prec 0$, we know $\B$ is invertible, and 
it is not hard to verify that 
$\argmax_{\norm{\dy} \le \max\left(h(\delta),h'(\delta)\right)} f(\x + \dx, \y+\dy)
= -\B^{-1}\C\trans \dx + o(\norm{\dx})$. Since $(\x,\y)$ is a local minimax point, we have
\begin{align*}
	0 & \le \max_{\norm{\dy} \le h(\delta)} f(\x+\delta_\x,\y+\delta_\y) - f(\x, \y) \leq \max_{\norm{\delta_y}\leq \max\left(h(\delta),h'(\delta)\right)} f(\x+\delta_\x,\y+\delta_\y) - f(\x, \y)\\
	& = \frac{1}{2} \dx\trans (\A - \C\B^{-1}\C\trans) \dx + o(\norm{\dx}^2).
\end{align*}
This equation holds for any $\dx$, which finishes the proof.
\end{proof}

\PROPstrictminimax*

\begin{proof}
Again denote $\A \defeq \phess_{\x\x} f(\x, \y)$, $\B \defeq \phess_{\y\y} f(\x, \y)$ and $\C \defeq \phess_{\x\y} f(\x, \y)$. Since $(\x, \y)$ is a stationary point, and $\B \prec 0$, it is clear that $\y$ is the local maximum of $f(\x, \cdot)$. On the other hand, pick 
$\dy^\dagger = \B^{-1}\C\trans \dx$ and letting $h(\delta) = \norm{\B^{-1}\C\trans} \delta$, 
we know that when $\norm{\dx} \le \delta$, we have $\norm{\dy^\dagger} \le h(\delta)$, thus
\begin{align*}
\max_{\norm{\dy} \le h(\delta)} f(\x + \dx, \y+\dy) - f(\x, \y)
\ge& f(\x + \dx, \y+\dy^\dagger) - f(\x, \y)\\
=&\frac{1}{2} \dx\trans (\A - \C\B^{-1}\C\trans) \dx + o(\norm{\dx}^2) >0,
\end{align*}
which finishes the proof.
\end{proof}

\PROPnosurrogate*

\begin{proof}
Consider the function $f(x, y) = 0.2 xy - \cos(y)$ in the region $[-1, 1] \times [-2\pi, 2\pi]$ as shown in Figure \ref{fig:surr}.
Clearly, the gradient is equal to $(0.2y, 0.2x + \sin(y))$. And, for any fixed $x$, there are only two maxima $y^\star(x)$ satisfying $0.2x + \sin(y^\star) = 0$ where $y^\star_1(x) \in (-3\pi/2, -\pi/2)$ and $y^\star_2(x) \in (\pi/2, 3\pi/2)$. On the other hand, $f(x, y^\star_1(x))$ is monotonically decreasing with respect to $x$, while $f(x, y^\star_2(x))$ is monotonically increasing, with $f(0, y^\star_1(0)) = f(0, y^\star_2(0))$ by symmetry. It is not hard to check $y^\star_1(0) = -\pi$ and $y^\star_2(0) = \pi$. Therefore, $(0, -\pi)$ and $(0, \pi)$ are two global solutions of the minimax problem. However, the gradients at both points are not $0$, thus they are not stationary points. By Proposition \ref{prop:first_nece} they are also not local minimax points.
\end{proof}

\LEMlocalMMnoexist*

\begin{proof}
Consider a two-dimensional function $f(x, y) = y^2 - 2xy$ on $[-1, 1] \times [-1, 1]$. Suppose $(x^\star, y^\star)$ is a local minimax point, then at least $y^\star$ is a local maximum of $f(x^\star, \cdot)$, which restricts the possible local minimax points to be within the set $[-1, 1)\times\{1\}$ or $(-1, 1]\times\{-1\}$. It is easy to check that no point in either set is local minimax.
\end{proof}

\THMgoodsurr*

\begin{proof}
Denote $\A \defeq \phess_{\x\x} f(\x, \y)$, $\B \defeq \phess_{\y\y} f(\x, \y)$, $\C \defeq \phess_{\x\y} f(\x, \y)$, $\g_\x \defeq \grad_{\x} f(\x, \y)$ and $\g_\y \defeq \grad_{\y} f(\x, \y)$. Let $(\x,\y)$ be a global minimax point. Since $\y$ is the global argmax of $f(\x, \cdot)$ and locally strongly concave, we know $\g_\y = 0$ and $\B \prec 0$. Let us now consider a second-order Taylor approximation of $f$ around $(\x,\y)$:
    \begin{align*}
        f(\x+\dx,\y+\dy) &= f(\x,\y) + \g_\x\trans\dx + \frac{1}{2}\dx\trans \A \dx + \dx\trans \C \dy + \frac{1}{2}\dy\trans \B \dy + o(\norm{\dx}^2 + \norm{\dy}^2).
        \end{align*}
    Since by hypothesis, $\B \prec 0$, we see that when $\norm{\delta_x}$ is sufficiently small, there is a unique $\dy^\star(\dx)$ so that $\y+\dy^\star(\dx)$
    is a local maximum of $f(\x+\dx, \cdot)$, where $\dy^\star (\dx) = -\B^{-1}\C\trans \dx + o(\norm{\dx})$. It is clear that $\norm{\dy^\star (\dx)} \le (\norm{\B^{-1}\C\trans} + 1) \norm{\dx}$ for sufficiently small $\norm{\dx}$. Let $h(\delta) = (\norm{\B^{-1}\C\trans} + 1) \delta$, we know for small enough $\delta$:
    \begin{equation*}
    f(\x+\dx, \y + \dy^\star(\dx)) = \max_{\norm{\dy} \le h(\delta)} f(\x + \dx, \y + \dy).
    \end{equation*}
    Finally, since by assumption for any $f(\x, \cdot)$ all local maxima are global maxima and $\x$ is the global min of $\max_{\y} f(\x, \y)$, we know:
    \begin{equation*}
    f(\x, \y) \le \max_{\y'} f(\x + \dx, \y') = f(\x+\dx, \y + \dy^\star(\dx)) = \max_{\norm{\dy} \le h(\delta)} f(\x + \dx, \y + \dy),
    \end{equation*}
    which finishes the proof.
\end{proof}

%% file: pf_prop.tex

\section{Proofs for Section \ref{sec:gda}}
In this section, we provides proofs for the propositions and theorems presented in Section \ref{sec:gda}.

\PROPstable*

\begin{proof}
Considering the GDA dynamics with step size $\eta$, the Jacobian matrix of this dynamic system is $\I + \eta \J_{\gamma}$ whose eigenvalues are $\{1+\eta\lambda_i\}$. Therefore, $(\x, \y)$ is a strict linearly stable point if and only if $\rho(\I + \eta \J_{\gamma}) <1$, that is $|1 + \eta\lambda_i| <1$ for all $i$. When taking $\eta \rightarrow 0$, this is equivalent to $\Re(\lambda_i) <0 $ for all $i$.
\end{proof}

\PROPlimitnash*
\begin{proof}
\citet{daskalakis2018limit} showed the proposition holds for $1$-GDA. For completeness, here we show how a similar proof goes through for $\gamma$-GDA for general $\gamma$. Let $\epsilon = 1/\gamma$, and denote $\A \defeq \phess_{\x\x} f(\x, \y)$, $\B \defeq \phess_{\y\y} f(\x, \y)$ and $\C \defeq \phess_{\x\y} f(\x, \y)$.

To prove the statement $\ttt{Local\_Nash} \subset \gamma\ttt{-GDA}$, we note that by definition $(\x, \y)$ is a strict linear stable point of $1/\epsilon$-GDA if the real part of the eigenvalues of Jacobian matrix 
\begin{equation*}
J_{\epsilon} \defeq \begin{pmatrix}
-\epsilon \A& -\epsilon \C \\
\C\trans& \B
\end{pmatrix}
\end{equation*}
satisfy that $\Re(\lambda_i) <0$ for all $1\le i \le d_1 + d_2$.
We first note that:
\begin{equation*}
\tilde{J}_{\epsilon} \defeq
\begin{pmatrix}
\B &\sqrt{\epsilon}\C\trans\\
-\sqrt{\epsilon} \C &-\epsilon \A
\end{pmatrix}
= \U J_\epsilon \U^{-1}, 
\text{~where~} \U = 
\begin{pmatrix}
0& \sqrt{\epsilon} \I \\
\I& 0
\end{pmatrix}.
\end{equation*}
Thus, the eigenvalues of $\tilde{J}_\epsilon$ and $J_\epsilon$ are the same. We can also decompose:
\begin{equation*}
\tilde{J}_\epsilon = \P + \Q, \text{~where~}
\P \defeq \begin{pmatrix}
\B &\\
&-\epsilon \A
\end{pmatrix}, 
\Q \defeq \begin{pmatrix} 0 &\sqrt{\epsilon}\C\trans\\
-\sqrt{\epsilon} \C & 0
\end{pmatrix}
\end{equation*}
If $(\x, \y)$ is a strict local \psne, then $\A \succ 0, \B \prec 0$, $\P$ is a negative definite symmetric matrix, and $\Q$ is anti-symmetric matrix: $\Q = -\Q\trans$.
For any eigenvalue $\lambda$ if $\tilde{J}_\epsilon$, assume $\w$ is the associated eigenvector. That is, $\tilde{J}_\epsilon \w = \lambda \w$, also let $\w = \x + i\y$ where $\x$ and $\y$ are real vectors, and $\bw$ be the complex conjugate of vector $\w$. Then:
\begin{align*}
\Re(\lambda) &= [\bw\trans \tilde{J}_\epsilon \w + \w\trans \tilde{J}_\epsilon \bw]/2
= [(\x - i\y)\trans \tilde{J}_\epsilon (\x + i\y) + (\x + i\y)\trans \tilde{J}_\epsilon (\x - i\y)]/2\\
&= \x\trans \tilde{J}_\epsilon \x + \y\trans \tilde{J}_\epsilon \y
= \x\trans \P \x + \y\trans \P \y
+ \x\trans \Q \x + \y\trans \Q \y.
\end{align*}
Since $\P$ is negative definite, we have $\x\trans \P \x + \y\trans \P \y<0$.
Meanwhile, since $\Q$ is antisymmtric $\x\trans \Q \x = 
\x\trans \Q\trans \x = 0$ and $\y\trans \Q \y  = \y\trans \Q\trans \y = 0$. This proves $\Re(\lambda)<0$; that is, $(\x, \y)$ is a strict linear stable point of $1/\epsilon$-GDA.

To prove the statement $\gamma\ttt{-GDA} \not\subset \ttt{Local\_Nash}$, since $\epsilon$ is also fixed, we consider the function $f(x, y) = x^2 + 2\sqrt{\epsilon} xy + (\epsilon/2) y^2$. It is easy to see that $(0, 0)$ is a fixed point of $1/\epsilon$-GDA, with Hessian $A = 2, B = \epsilon, C = 2\sqrt{\epsilon}$. Thus the Jacobian matrix
\begin{equation*}
J_{\epsilon} \defeq \begin{pmatrix}
-2\epsilon& -2\epsilon^{3/2}  \\
2\epsilon^{1/2} & \epsilon
\end{pmatrix}
\end{equation*}
has two eigenvalues $\epsilon(-1 \pm i\sqrt{7})/2$. Therefore, $\Re(\lambda_1) = \Re(\lambda_2) <0 $, which implies $(0, 0)$ is a strict linear stable point. However $B = \epsilon >0$, thus it is not a strict local \psne.
\end{proof}

\PROPlimitminimax*

\begin{proof}
Let $\epsilon = 1/\gamma$, and denote $\A \defeq \phess_{\x\x} f(\x, \y)$, 
$\B \defeq \phess_{\y\y} f(\x, \y)$ and $\C \defeq \phess_{\x\y} f(\x, \y)$.

To prove the first statement, $\ttt{Local\_Minimax} \not \subset \gamma\ttt{-GDA}$, since $\epsilon$ is also fixed, we consider the function $f(x, y) = -x^2 + 2\sqrt{\epsilon} xy - (\epsilon/2) y^2$. It is easy to see $(0, 0)$ is a fixed point of $1/\epsilon$-GDA, and Hessian $A = -2, B = -\epsilon, C = 2\sqrt{\epsilon}$.
It is easy to verify that $B < 0$ and $A - CB^{-1}C = 2 >0$, thus $(0, 0)$ is a local \minimax. Consider, however, the Jacobian matrix of $1/\epsilon$-GDA:
\begin{equation*}
J_{\epsilon} \defeq \begin{pmatrix}
2\epsilon& -2\epsilon^{3/2}  \\
2\epsilon^{1/2} & -\epsilon
\end{pmatrix}.
\end{equation*}
We know the two eigenvalues are $\epsilon(1 \pm i\sqrt{7})/2$. Therefore, $\Re(\lambda_1) = \Re(\lambda_2) >0 $, which implies $(0, 0)$ is not a strict linear stable point.

To prove the second statement, $\gamma\ttt{-GDA}\not \subset \ttt{Local\_Minimax} \cup \ttt{Local\_Maximin}$,  since $\epsilon$ is also fixed, we consider the function $f(\x, \y) = x_1^2 + 2\sqrt{\epsilon} x_1y_1 + (\epsilon/2) y_1^2 - x_2^2/2 + 2\sqrt{\epsilon} x_2y_2 - \epsilon y_2^2$. It is easy to see $(\zero, \zero)$ is a fixed point of $1/\epsilon$-GDA, with Hessian $\A = \diag(2, -1), \B = \diag(\epsilon, -2\epsilon), \C = 2\sqrt{\epsilon}\cdot\diag(1, 1)$. Thus the Jacobian matrix
\begin{equation*}
J_{\epsilon} \defeq \begin{pmatrix}
-2\epsilon& 0 & -2\epsilon^{3/2} & 0\\
0 &\epsilon& 0 & -2\epsilon^{3/2}\\
2\epsilon^{1/2} & 0 &\epsilon & 0\\
0 &2\epsilon^{1/2}& 0 & -2\epsilon
\end{pmatrix}
\end{equation*}
has four eigenvalues $\epsilon(-1 \pm i\sqrt{7})/2$ (each with multiplicity of $2$). Therefore, $\Re(\lambda_i) <0 $ for $1\le i \le 4$, which implies $(\zero, \zero)$ is a strict linear stable point. However, $\B$ is not negative definite, thus $(\zero, \zero)$ is not a strict local \minimax; similarly, $\A$ is also not positive definite, thus $(\zero, \zero)$ is not a strict local maximin point.
\end{proof}

%% file: pf_gda.tex

\THMgdamain*

\begin{proof}
For simplicity, denote $\A \defeq \phess_{\x\x} f(\x, \y)$, 
$\B \defeq \phess_{\y\y} f(\x, \y)$ and $\C \defeq \phess_{\x\y} f(\x, \y)$. Let $\epsilon = 1/\gamma$. Considering a sufficiently small $\epsilon$ (i.e., sufficiently large $\gamma$), we know the Jacobian $J$ of $1/\epsilon$-GDA at $(\x, \y)$ is
\begin{equation*}
J_{\epsilon} \defeq \begin{pmatrix}
-\epsilon \A& -\epsilon \C \\
\C\trans& \B
\end{pmatrix}.
\end{equation*}
According to Lemma \ref{lem:eigen_J}, for sufficiently small $\epsilon$, 
$J_\epsilon$ has $d_1 + d_2$ complex eigenvalues $\{\lambda_i\}_{i=1}^{d_1+d_2}$ with following form: 
\begin{align}
&|\lambda_{i} + \epsilon \mu_i| = o(\epsilon) &1 \le i \le d_1 \nn\\
&|\lambda_{i+d_1}- \nu_i| =  o(1), &1 \le i \le d_2, \label{eq:eigen_structure}
\end{align}
where $\{\mu_i\}_{i=1}^{d_1}$ and $\{\nu_i\}_{i=1}^{d_2}$ are the eigenvalues of matrices $\A - \C\B^{-1}\C\trans$ and $\B$ respectively.
Now we are ready to prove the three inclusion statement in Theorem \ref{thm:main} separately. 

First, we have that $\underline{\infty\ttt{-GDA}} \subset \overline{\infty\ttt{-GDA}}$ always hold by their definitions.

Second, for $\ett{Local\_Minimax} \subset \underline{\infty\ttt{-GDA}}$ statement, if $(\x, \y)$ is strict local \minimax, then by its definition:
\begin{equation*}
\B \prec 0, \quad \text{and} \quad \A - \C\B^{-1}\C\trans \succ 0.
\end{equation*}
By Eq.~\eqref{eq:eigen_structure}, we know there exists sufficiently small $\epsilon_0$ such that for any $\epsilon <\epsilon_0$, the real part $\Re(\lambda_i) <0$; i.e., $(\x, \y)$ is a strict linear stable point of $1/\epsilon-$GDA.

Finally, for $\overline{\infty\ttt{-GDA}} \subset \ett{Local\_Minimax} \cup \{(\x, \y)| (\x, \y)$ is stationary and $\B$ is degenerate$\}$ statement, if $(\x, \y)$ is a strict linear stable point of $1/\epsilon-$GDA for sufficiently small $\epsilon$, 
then for any $i$, the real part of eigenvalue of $J_\epsilon$: $\Re(\lambda_i) <0$. By Eq.\eqref{eq:eigen_structure}, if $\B$ is invertible, this implies:
\begin{equation*}
\B \prec 0, \quad \text{and} \quad \A - \C\B^{-1}\C\trans \succeq 0.
\end{equation*}
Suppose that the matrix $\A - \C\B^{-1}\C\trans$ has an eigenvalue $0$. This implies the existence of unit vector $\w$ so that $(\A - \C\B^{-1}\C\trans) \w = 0$. It is not hard to verify that then $J_\epsilon \cdot (\w, -\B^{-1}\C\trans \w)\trans = 0$. This implies $J_{\epsilon}$ has a zero eigenvalue, which contradicts the fact that $\Re(\lambda_i) <0$ for all $i$. Therefore, we can conclude $\A - \C\B^{-1}\C\trans \succ 0$, and $(\x, \y)$ is a strict local \minimax.
\end{proof}

\begin{lemma}\label{lem:eigen_J}
For any symmetric matrix $\A \in \R^{d_1\times d_1}$, $\B \in \R^{d_2\times d_2}$,
and any rectangular matrix $\C \in \R^{d_1\times d_2}$, assume $\B$ is nondegenerate. Then, matrix
\begin{equation*}
\begin{pmatrix}
-\epsilon \A& -\epsilon \C \\
\C\trans& \B
\end{pmatrix}
\end{equation*}
has $d_1 + d_2$ complex eigenvalues $\{\lambda_i\}_{i=1}^{d_1+d_2}$ with following form for sufficient small $\epsilon$: 
\begin{align*}
&|\lambda_{i} + \epsilon \mu_i| = o(\epsilon) &1 \le i \le d_1\\
&|\lambda_{i+d_1}- \nu_i| =  o(1), &1 \le i \le d_2,
\end{align*}
where $\{\mu_i\}_{i=1}^{d_1}$ and $\{\nu_i\}_{i=1}^{d_2}$ are the eigenvalues of matrices $\A - \C\B^{-1}\C\trans$ and $\B$ respectively.
\end{lemma}

\begin{proof}
By definition of eigenvalues, $\{\lambda_i\}_{i=1}^{d_1+d_2}$ are the roots of characteristic polynomial:
\begin{equation*}
p_{\epsilon}(\lambda) \defeq \det
\begin{pmatrix}
\lambda\I+\epsilon \A& \epsilon \C \\
-\C\trans& \lambda\I - \B
\end{pmatrix}.
\end{equation*}
We can expand this polynomial as:
\begin{equation*}
p_{\epsilon}(\lambda) = p_0(\lambda) + \sum_{i=1}^{d_1+d_2} \epsilon^{i} p_i(\lambda), 
\qquad p_0(\lambda) =\lambda^{d_1} \cdot \det(\lambda \I - \B).
\end{equation*}
Here, $p_i$ are polynomials of order at most $d_1 + d_2$. It is clear that the roots of $p_0$ are zero (with multiplicity $d_1$) and $\{\nu_i\}_{i=1}^{d_2}$.
According to Lemma \ref{lem:stability_roots}, we know the roots of $p_\epsilon$ satisfy:
\begin{align}
&|\lambda_{i}| = o(1) &1 \le i \le d_1 \label{eq:small_roots}\\
&|\lambda_{i+d_1}- \nu_i| =  o(1), &1 \le i \le d_2 \nn.
\end{align}
Since $\B$ is non-degenerate, we know when $\epsilon$ is small enough, $\lambda_{1} \ldots \lambda_{d_1}$ are very close to zero while $\lambda_{d_1+1} \ldots \lambda_{d_1+d_2}$ have modulus at least $\Omega(1)$. To provide the sign information of the first $d_1$ roots, we proceed to a lower-order characterization.

Reparametrizing $\lambda  =\epsilon \theta$, we have:
\begin{equation*}
p_{\epsilon}(\epsilon\theta) = \det
\begin{pmatrix}
\epsilon\theta\I+\epsilon \A& \epsilon \C \\
-\C\trans& \epsilon\theta\I - \B
\end{pmatrix}
= \epsilon^{d_1} \det
\begin{pmatrix}
\theta\I+ \A&  \C \\
-\C\trans& \epsilon\theta\I - \B
\end{pmatrix}.
\end{equation*}

Therefore, we know $q_{\epsilon}(\theta) \defeq p_{\epsilon}(\epsilon\theta)/\epsilon^{d_1}$
is still a polynomial, and has a polynomial expansion:
\begin{equation*}
q_{\epsilon}(\theta) = q_0(\theta) + \sum_{i=1}^{d_2} \epsilon^{i} q_i(\lambda), 
\qquad q_0(\theta) =\det
\begin{pmatrix}
\theta\I+ \A&  \C \\
-\C\trans& - \B
\end{pmatrix}
\end{equation*}
It is also clear that the polynomials $q_\epsilon$ and $p_\epsilon$ have the same roots up to $\epsilon$ scaling. Furthermore, we have following factorization:
\begin{equation*}
\begin{pmatrix}
\theta\I+ \A& \C \\
-\C\trans& -\B
\end{pmatrix}
=\begin{pmatrix}
\theta\I+ \A - \C  \B^{-1} \C\trans & \C \\
0 & -\B
\end{pmatrix}
\begin{pmatrix}
\I & 0 \\ \B^{-1}\C\trans & \I 
\end{pmatrix}.
\end{equation*}
Since $\B$ is non-degenerate, we have $\det(\B) \neq 0$, and 
\begin{equation*}
q_0(\theta) = (-1)^{d_2}\det(\B)\det(\theta\I+ \A - \C  \B^{-1} \C\trans)
\end{equation*}
$q_0$ is $d_1$-order polynomial having roots $\{\mu_i\}_{i=1}^{d_1}$, which are the eigenvalues of matrices $\A - \C\B^{-1}\C\trans$. According to Lemma \ref{lem:stability_roots}, we know $q_{\epsilon}$ has at least $d_1$ roots so that $|\theta_i + \mu_i| \le o(1)$. This implies that $d_1$ roots of $p_\epsilon$ are such that:
\begin{equation*}
|\lambda_{i} + \epsilon \mu_i| = o(\epsilon) \qquad 1 \le i \le d_1
\end{equation*}
By Eq.~\eqref{eq:small_roots}, we know $p_\epsilon$ has exactly $d_1$ roots which are of $o(1)$ scaling. This finishes the proof.
\end{proof}


\begin{lemma}[Continuity of roots of polynomials, \cite{zedek1965continuity}]
\label{lem:stability_roots}
Given a polynomial $p_n(z)\defeq \sum_{k=0}^n a_k z^k$, $a_n \neq 0$, an integer $m\ge n$ and a number $\epsilon>0$, there exists a number $\delta >0$ such that whenever the $m+1$ complex numbers $b_k, 0\le k \le m$, satisfy the inequalities
\begin{equation*}
|b_k - a_k| < \delta \text{~~for~~} 0\le k \le n, 
\text{~~and~~}
|b_k| <\delta \text{~~for~~} n+1\le k\le m,
\end{equation*}
then the roots $\beta_k, 1\le k\le m$ of the polynomial $q_m(z) \defeq \sum_{k=0}^m b_k z^k$ can be labeled in such a way as to satisfy, with respect to the zeros $\alpha_k, 1\le k\le n$ of $p_n(z)$, the inequalities
\begin{equation*}
|\beta_k - \alpha_k| <\epsilon \text{~~for~~} 1\le k\le n,
\text{~~and~~}
|\beta_k| >1/\epsilon \text{~~for~~} n+1\le k\le m.
\end{equation*}
\end{lemma}


%% file: pf_oracle.tex

\section{Proofs for Section \ref{sec:oracle}}

In this section, we prove Theorem \ref{thm:max_oracle}.

\begin{proof}[Proof of Theorem \ref{thm:max_oracle}]
	The proof of this theorem mainly follows the proof of Theorem~2.1 from~\cite{davis2018stochastic}. The only difference is that $y_t$ in Algorithm~\ref{algo:GDMO} is only an approximate maximizer and not exact maximizer. However, the proof goes through fairly easily with an additional error term. 
	
	We first note an important equation for the gradient of Moreau envelope.
	\begin{align}
	\nabla \phi_{\lambda}(\x) = \lambda^{-1}\left(\x - \argmin_{\xtilde} \left(\phi(\xtilde) + \frac{1}{2\lambda} \norm{\x - \xtilde}^2\right)\right). \label{eq:moreaugrad}
	\end{align}
	We also observe that since $f(\cdot)$ is $\ell$-smooth and $y_t$ is an approximate maximizer for $\x_t$, we have that any $\x_{t}$ from Algorithm~\ref{algo:GDMO} and $\xtilde$ satisfy
	\begin{align}
		\phi(\xtilde) \geq f(\xtilde, \y_t) & \geq f(\x_t,\y_t) + \iprod{\nabla_{\x} f(\x_t,\y_t)}{\xtilde-\x_t} - \frac{\ell}{2} \norm{\xtilde - \x_t}^2 \nonumber \\
		& \geq \phi(\x_t) - \epsilon + \iprod{\nabla_{\x} f(\x_t,\y_t)}{\xtilde-\x_t} - \frac{\ell}{2} \norm{\xtilde - \x_t}^2.\label{eq:subgrad}
	\end{align}
	Let $\xhat_t \defeq \argmin_{\x} \phi(\x) + {\ell}{} \norm{\x - \x_t}^2$. We have:
	\begin{align*}
		\phiell\left(\x_{t+1}\right) &\leq \phi(\xhat_t) + {\ell}{} \norm{\x_{t+1} - \xhat_t}^2 \\
		&\leq \phi(\xhat_t) + {\ell}{} \norm{\x_{t} - \eta \nabla_{\x} f(\x_t,\y_t) - \xhat_t}^2 \\
		&\leq \phi(\xhat_t) + {\ell}{} \norm{\x_{t} - \xhat_t}^2 + {2 \ell\eta}{}\iprod{\nabla_{\x} f(\x_t,\y_t) }{\xhat_t - \x_{t}} + {\eta^2\ell}{} \norm{\nabla_{\x} f(\x_t,\y_t)}^2\\
		&\leq \phiell(\x_t) + {2 \eta\ell}{}\iprod{\nabla_{\x} f(\x_t,\y_t) }{\xhat_t - \x_{t}} + {\eta^2\ell}{} \norm{\nabla_{\x} f(\x_t,\y_t)}^2\\
		&\leq \phiell(\x_t) + {2 \eta\ell}{}\left(\phi(\xhat_t) - \phi(\x_t) + \epsilon + \frac{\ell}{2} \norm{\x_{t} - \xhat_t}^2\right) + {\eta^2\ell}{} L^2,
	\end{align*}
	where the last line follows from~\eqref{eq:subgrad}. Taking a telescopic sum over $t$, we obtain
	\begin{align*}
		\phiell(\x_{T}) \leq \phiell(\x_0) + 2 \eta \ell \sum_{t=0}^{T} \left(\phi(\xhat_t) - \phi(\x_t) + \epsilon + \frac{\ell}{2} \norm{\x_{t} - \xhat_t}^2\right) + {\eta^2 \ell L^2 T}{}
	\end{align*}
	Rearranging this, we obtain
	\begin{align}
		\frac{1}{T+1} \sum_{t=0}^{T} \left(\phi(\x_t) - \phi(\xhat_t) - \frac{\ell}{2} \norm{\x_{t} - \xhat_t}^2\right) \leq \epsilon + \frac{\phiell(\x_0) - \min_{\x} \phi(\x)}{2 \eta \ell T} + \frac{\eta L^2}{2}.\label{eq:avg-subgrad}
	\end{align}
	Since $\phi(\x) + \ell \norm{\x - \x_t}^2$ is ${\ell}$-strongly convex, we have
	\begin{align*}
		&\phi(\x_t) - \phi(\xhat_t) - \frac{\ell}{2} \norm{\x_{t} - \xhat_t}^2 \\ &\geq \phi(\x_t) + \ell \norm{\x_t - \x_t}^2 - \phi(\xhat_t) - \ell \norm{\xhat_t - \x_t}^2 + \frac{\ell}{2} \norm{\x_{t} - \xhat_t}^2 \\
		&= \left(\phi(\x_t) + \ell \norm{\x_t - \x_t}^2 - \min_{\x} \phi(\x) + \ell \norm{\x - \x_t}^2 \right) + \frac{\ell}{2} \norm{\x_{t} - \xhat_t}^2 \\
		&\geq {\ell}{} \norm{\x_{t} - \xhat_t}^2 = \frac{1}{4\ell} \norm{\nabla \phiell(\x_{t})}^2,
	\end{align*}
	where we used~\eqref{eq:moreaugrad} in the last step. Plugging this in~\eqref{eq:avg-subgrad} proves the result.
\end{proof}

%% file: main.bbl
\begin{thebibliography}{37}
\providecommand{\natexlab}[1]{#1}
\providecommand{\url}[1]{\texttt{#1}}
\expandafter\ifx\csname urlstyle\endcsname\relax
  \providecommand{\doi}[1]{doi: #1}\else
  \providecommand{\doi}{doi: \begingroup \urlstyle{rm}\Url}\fi

\bibitem[Adolphs et~al.(2018)Adolphs, Daneshmand, Lucchi, and
  Hofmann]{adolphs2018local}
Leonard Adolphs, Hadi Daneshmand, Aurelien Lucchi, and Thomas Hofmann.
\newblock Local saddle point optimization: A curvature exploitation approach.
\newblock \emph{arXiv preprint arXiv:1805.05751}, 2018.

\bibitem[Arora et~al.(2017)Arora, Ge, Liang, Ma, and
  Zhang]{arora2017generalization}
Sanjeev Arora, Rong Ge, Yingyu Liang, Tengyu Ma, and Yi~Zhang.
\newblock Generalization and equilibrium in generative adversarial nets (gans).
\newblock \emph{arXiv preprint arXiv:1703.00573}, 2017.

\bibitem[Bertsekas(2014)]{bertsekas2014constrained}
Dimitri~P Bertsekas.
\newblock \emph{Constrained optimization and Lagrange multiplier methods}.
\newblock Academic press, 2014.

\bibitem[Boumal et~al.(2016)Boumal, Voroninski, and Bandeira]{boumal2016non}
Nicolas Boumal, Vlad Voroninski, and Afonso Bandeira.
\newblock The non-convex burer-monteiro approach works on smooth semidefinite
  programs.
\newblock In \emph{Advances in Neural Information Processing Systems}, pages
  2757--2765, 2016.

\bibitem[Bruck~Jr(1977)]{bruck1977weak}
Ronald~E Bruck~Jr.
\newblock On the weak convergence of an ergodic iteration for the solution of
  variational inequalities for monotone operators in hilbert space.
\newblock \emph{Journal of Mathematical Analysis and Applications}, 61\penalty0
  (1):\penalty0 159--164, 1977.

\bibitem[Bubeck(2015)]{bubeck2015convex}
S{\'e}bastien Bubeck.
\newblock Convex optimization: Algorithms and complexity.
\newblock \emph{Foundations and Trends{\textregistered} in Machine Learning},
  8\penalty0 (3-4):\penalty0 231--357, 2015.

\bibitem[Cherukuri et~al.(2017)Cherukuri, Gharesifard, and
  Cortes]{cherukuri2017saddle}
Ashish Cherukuri, Bahman Gharesifard, and Jorge Cortes.
\newblock Saddle-point dynamics: conditions for asymptotic stability of saddle
  points.
\newblock \emph{SIAM Journal on Control and Optimization}, 55\penalty0
  (1):\penalty0 486--511, 2017.

\bibitem[Daskalakis and Panageas(2018)]{daskalakis2018limit}
Constantinos Daskalakis and Ioannis Panageas.
\newblock The limit points of (optimistic) gradient descent in min-max
  optimization.
\newblock In \emph{Advances in Neural Information Processing Systems}, pages
  9256--9266, 2018.

\bibitem[Davis and Drusvyatskiy(2018)]{davis2018stochastic}
Damek Davis and Dmitriy Drusvyatskiy.
\newblock Stochastic subgradient method converges at the rate $ o
  (k^\frac{-1}{4}) $ on weakly convex functions.
\newblock \emph{arXiv preprint arXiv:1802.02988}, 2018.

\bibitem[Evtushenko(1974)]{evtushenko1974some}
Yurii~Gavrilovich Evtushenko.
\newblock Some local properties of minimax problems.
\newblock \emph{Zhurnal Vychislitel'noi Matematiki i Matematicheskoi Fiziki},
  14\penalty0 (3):\penalty0 669--679, 1974.

\bibitem[Fiez et~al.(2019)Fiez, Chasnov, and Ratliff]{fiez2019convergence}
Tanner Fiez, Benjamin Chasnov, and Lillian~J Ratliff.
\newblock Convergence of learning dynamics in stackelberg games.
\newblock \emph{arXiv preprint arXiv:1906.01217}, 2019.

\bibitem[Ge et~al.(2017)Ge, Jin, and Zheng]{ge2017no}
Rong Ge, Chi Jin, and Yi~Zheng.
\newblock No spurious local minima in nonconvex low rank problems: A unified
  geometric analysis.
\newblock In \emph{International Conference on Machine Learning}, pages
  1233--1242, 2017.

\bibitem[Gidel et~al.(2018)Gidel, Hemmat, Pezeshki, Huang, Lepriol,
  Lacoste-Julien, and Mitliagkas]{gidel2018negative}
Gauthier Gidel, Reyhane~Askari Hemmat, Mohammad Pezeshki, Gabriel Huang, Remi
  Lepriol, Simon Lacoste-Julien, and Ioannis Mitliagkas.
\newblock Negative momentum for improved game dynamics.
\newblock \emph{arXiv preprint arXiv:1807.04740}, 2018.

\bibitem[Glicksberg(1952)]{glicksberg1952further}
Irving~L Glicksberg.
\newblock A further generalization of the kakutani fixed point theorem, with
  application to {N}ash equilibrium points.
\newblock \emph{Proceedings of the American Mathematical Society}, 3\penalty0
  (1):\penalty0 170--174, 1952.

\bibitem[Goodfellow et~al.(2014)Goodfellow, Pouget-Abadie, Mirza, Xu,
  Warde-Farley, Ozair, Courville, and Bengio]{goodfellow2014generative}
Ian Goodfellow, Jean Pouget-Abadie, Mehdi Mirza, Bing Xu, David Warde-Farley,
  Sherjil Ozair, Aaron Courville, and Yoshua Bengio.
\newblock Generative adversarial nets.
\newblock In \emph{Advances in neural information processing systems}, pages
  2672--2680, 2014.

\bibitem[Hazan(2016)]{hazan2016introduction}
Elad Hazan.
\newblock Introduction to online convex optimization.
\newblock \emph{Foundations and Trends{\textregistered} in Optimization},
  2\penalty0 (3-4):\penalty0 157--325, 2016.

\bibitem[Heusel et~al.(2017)Heusel, Ramsauer, Unterthiner, Nessler, and
  Hochreiter]{heusel2017gans}
Martin Heusel, Hubert Ramsauer, Thomas Unterthiner, Bernhard Nessler, and Sepp
  Hochreiter.
\newblock Gans trained by a two time-scale update rule converge to a local nash
  equilibrium.
\newblock In \emph{Advances in Neural Information Processing Systems}, pages
  6626--6637, 2017.

\bibitem[Hsieh et~al.(2018)Hsieh, Liu, and Cevher]{hsieh2018finding}
Ya-Ping Hsieh, Chen Liu, and Volkan Cevher.
\newblock Finding mixed nash equilibria of generative adversarial networks.
\newblock \emph{arXiv preprint arXiv:1811.02002}, 2018.

\bibitem[Kinderlehrer and Stampacchia(1980)]{kinderlehrer1980introduction}
David Kinderlehrer and Guido Stampacchia.
\newblock \emph{An introduction to variational inequalities and their
  applications}, volume~31.
\newblock Siam, 1980.

\bibitem[Korpelevich(1976)]{korpelevich1976extragradient}
GM~Korpelevich.
\newblock The extragradient method for finding saddle points and other
  problems.
\newblock \emph{Matecon}, 12:\penalty0 747--756, 1976.

\bibitem[Lin et~al.(2018)Lin, Liu, Rafique, and Yang]{lin2018solving}
Qihang Lin, Mingrui Liu, Hassan Rafique, and Tianbao Yang.
\newblock Solving weakly-convex-weakly-concave saddle-point problems as
  weakly-monotone variational inequality.
\newblock \emph{arXiv preprint arXiv:1810.10207}, 2018.

\bibitem[Madry et~al.(2017)Madry, Makelov, Schmidt, Tsipras, and
  Vladu]{madry2017towards}
Aleksander Madry, Aleksandar Makelov, Ludwig Schmidt, Dimitris Tsipras, and
  Adrian Vladu.
\newblock Towards deep learning models resistant to adversarial attacks.
\newblock \emph{arXiv preprint arXiv:1706.06083}, 2017.

\bibitem[Mazumdar and Ratliff(2018)]{mazumdar2018convergence}
Eric Mazumdar and Lillian~J Ratliff.
\newblock On the convergence of gradient-based learning in continuous games.
\newblock \emph{arXiv preprint arXiv:1804.05464}, 2018.

\bibitem[Mazumdar et~al.(2019)Mazumdar, Jordan, and
  Sastry]{mazumdar2019finding}
Eric~V Mazumdar, Michael~I Jordan, and S~Shankar Sastry.
\newblock On finding local nash equilibria (and only local nash equilibria) in
  zero-sum games.
\newblock \emph{arXiv preprint arXiv:1901.00838}, 2019.

\bibitem[Morgenstern and Von~Neumann(1953)]{morgenstern1953theory}
Oskar Morgenstern and John Von~Neumann.
\newblock \emph{Theory of games and economic behavior}.
\newblock Princeton university press, 1953.

\bibitem[Myerson(2013)]{myerson2013game}
Roger~B Myerson.
\newblock \emph{Game theory}.
\newblock Harvard university press, 2013.

\bibitem[Nagarajan and Kolter(2017)]{nagarajan2017gradient}
Vaishnavh Nagarajan and J~Zico Kolter.
\newblock Gradient descent gan optimization is locally stable.
\newblock In \emph{Advances in Neural Information Processing Systems}, pages
  5585--5595, 2017.

\bibitem[Nemirovski(1981)]{nemirovsky1981}
Arkadi Nemirovski.
\newblock Efficient methods for solving variational inequalities.
\newblock \emph{Ekonomika i Matem.Metody}, 17:\penalty0 344--359, 1981.

\bibitem[Nemirovski(2004)]{nemirovski2004prox}
Arkadi Nemirovski.
\newblock Prox-method with rate of convergence o (1/t) for variational
  inequalities with lipschitz continuous monotone operators and smooth
  convex-concave saddle point problems.
\newblock \emph{SIAM Journal on Optimization}, 15\penalty0 (1):\penalty0
  229--251, 2004.

\bibitem[Nemirovski and Yudin(1978)]{nemirovsky1978}
Arkadi Nemirovski and D.~Yudin.
\newblock Cesari convergence of the gradient method for approximation sad-dle
  points of convex-concave functions.
\newblock \emph{Doklady AN SSSRv}, 239:\penalty0 1056--1059, 1978.

\bibitem[Nouiehed et~al.(2019)Nouiehed, Sanjabi, Lee, and
  Razaviyayn]{nouiehed2019solving}
Maher Nouiehed, Maziar Sanjabi, Jason~D Lee, and Meisam Razaviyayn.
\newblock Solving a class of non-convex min-max games using iterative first
  order methods.
\newblock \emph{arXiv preprint arXiv:1902.08297}, 2019.

\bibitem[Omidshafiei et~al.(2017)Omidshafiei, Pazis, Amato, How, and
  Vian]{omidshafiei2017deep}
Shayegan Omidshafiei, Jason Pazis, Christopher Amato, Jonathan~P How, and John
  Vian.
\newblock Deep decentralized multi-task multi-agent reinforcement learning
  under partial observability.
\newblock \emph{arXiv preprint arXiv:1703.06182}, 2017.

\bibitem[Rafique et~al.(2018)Rafique, Liu, Lin, and Yang]{rafique2018non}
Hassan Rafique, Mingrui Liu, Qihang Lin, and Tianbao Yang.
\newblock Non-convex min-max optimization: Provable algorithms and applications
  in machine learning.
\newblock \emph{arXiv preprint arXiv:1810.02060}, 2018.

\bibitem[Rockafellar(2015)]{rockafellar2015convex}
Ralph~Tyrell Rockafellar.
\newblock \emph{Convex analysis}.
\newblock Princeton university press, 2015.

\bibitem[Sion et~al.(1958)]{sion1958general}
Maurice Sion et~al.
\newblock On general minimax theorems.
\newblock \emph{Pacific Journal of mathematics}, 8\penalty0 (1):\penalty0
  171--176, 1958.

\bibitem[von Neumann(1928)]{neumann1928theorie}
J~von Neumann.
\newblock Zur theorie der gesellschaftsspiele.
\newblock \emph{Mathematische annalen}, 100\penalty0 (1):\penalty0 295--320,
  1928.

\bibitem[Zedek(1965)]{zedek1965continuity}
Mishael Zedek.
\newblock Continuity and location of zeros of linear combinations of
  polynomials.
\newblock \emph{Proceedings of the American Mathematical Society}, 16\penalty0
  (1):\penalty0 78--84, 1965.

\end{thebibliography}
